\documentclass[11pt]{article}
\usepackage{setspace}
\usepackage{amsmath,amssymb}
\usepackage{amsthm}
\usepackage{algorithm, algpseudocode}
\usepackage{url}
\usepackage{fullpage}
\usepackage{color}
\usepackage{graphicx}

\newtheorem{claim}{Claim}[section]
\newtheorem{lemma}[claim]{Lemma}

\newtheorem{theorem}{Theorem}
\newtheorem{proposition}[claim]{Proposition}

\theoremstyle{definition}
\newtheorem{remark}[claim]{Remark}

\def\ttau{\tilde{\tau}}

\def\sML{\mbox{\tiny\sc ML}}
\def\Loss{{\sf Loss}}
\def\orho{\overline{\rho}}
\def\oM{\overline{M}}
\def\<{\langle}
\def\>{\rangle}
\def\cN{\mathcal N}
\def\prob{{\mathbb P}}

\def\naturals{{\mathbb N}}
\def\E{{\mathbb E}} 

\def\bM{{\bf M}}

\def\bE{{\bf \Xi}}

\def\reals{\mathbb{R}}

\def\sT{{\sf T}}

\def\bx{\mathbf{x}}
\def\bu{\mathbf{u}}
\def\bv{\mathbf{v}}
\def\bo{\mathbf{o}}

\def\hbv{\mathbf{\widehat{v}}}

\def\bg{\mathbf{g}}

\def\bz{\mathbf{z}}

\def\bw{\mathbf{w}}
\def\hbw{\widehat{\mathbf{w}}}

\def\bvz{\mathbf{v_0}}
\def\bwz{\mathbf{w_0}}
\def\buz{\mathbf{u_0}}
\def\bX{\mathbf{X}}
\def\btX{\mathbf{\widetilde{X}}}
\def\bN{\mathbf{N}}
\def\bM{\mathbf{M}}

\def\bY{\mathbf{Y}}
\def\vec{{\sf{vec}}}

\def\bZ{\mathbf{Z}}
\def\bG{\mathbf{G}}

\def\sU{{\sf{U}}}

\def\cV{\mathcal{V}}

\def\bW{\mathbf{W}}
\def\cW{\mathcal{W}}

\def\bbS{\mathbb{S}}

\def\E{\mathbb{E}}
\def\normal{{\sf N}}

\def\ok{{^{\otimes k}}}

\def\M{{\sf Mat}}

\def\P{{\sf P}_\succeq}

\def\id{{\rm I}}
\def\eps{\varepsilon}

\def\trans{^\sT}

\def\ons{{\sf b}}

\def\cX{{\cal X}}
\def\cY{{\cal Y}}

\def\by{{\mathbf y}}

\author{Andrea~Montanari \footnote{Department of Electrical
    Engineering and Department of Statistics, Stanford University}
\, and\, Emile Richard\footnote{Department of Electrical
    Engineering, Stanford University}}
\title{A statistical model for tensor PCA}

\begin{document}

\maketitle

\begin{abstract}
We consider the Principal Component Analysis problem for large tensors
of arbitrary order $k$ under a single-spike (or rank-one plus noise)
model.
On the one hand, we use information theory, and recent results in
probability theory, to establish necessary and sufficient conditions
under  which the principal component can be estimated using unbounded
computational resources. It turns out that this is possible as soon as
the signal-to-noise ratio $\beta$ becomes larger than $C\sqrt{k\log k}$ (and
in particular $\beta$ can remain bounded as the problem dimensions increase).

On the other hand, we analyze several polynomial-time estimation
algorithms, based on tensor unfolding, power iteration and message
passing ideas from graphical
models.  We show that, unless the signal-to-noise ratio diverges in
the system dimensions, none of these approaches succeeds. This is
possibly related to a fundamental limitation of computationally
tractable estimators for this problem.

We discuss various initializations for tensor power iteration, and show that a tractable initialization based on the spectrum of the matricized  tensor outperforms
significantly  baseline methods, statistically and computationally. Finally, we consider the case in
which additional side information is available about the unknown
signal.
We characterize the amount of side information that allows the
iterative algorithms to converge to a good estimate. 
\end{abstract}

\section{Introduction}
Given a data matrix $\bX$, Principal Component Analysis (PCA) can be
regarded as a `denoising' technique that replaces $\bX$ by its closest
rank-one approximation. This optimization problem can be solved
efficiently, and its statistical properties are well-understood.
The generalization of PCA to tensors is motivated by problems in which it
is important to exploit higher order moments, or data elements are
naturally given more than two indices. Examples include  topic
modeling \cite{anandkumar12}, video processing,
 collaborative filtering in presence of temporal/context
 information, community detection \cite{anandkumar2013tensor},  spectral hypergraph theory and hyper-graph
 matching \cite{Duchenne09}. 
Further, finding a rank-one approximation to a tensor is a bottleneck for  tensor-valued optimization
algorithms using conditional gradient type of schemes. 
While tensor factorization is NP-hard \cite{hillar2013most}, this does not
necessarily imply intractability for natural statistical
models. Over the last ten years, it was repeatedly observed that
either convex optimization or greedy methods yield optimal solutions
to statistical problems that are intractable from a worst case
perspective (well-known examples include sparse regression \cite{donoho2003optimally,tropp2004greed,candes2007dantzig} and
low-rank matrix completion \cite{candes2009exact,keshavan2010matrix}).

In order to investigate the fundamental  tradeoffs between
computational resources and statistical power in
tensor PCA, we consider the simplest possible model where this arises,
whereby an  unknown unit vector $\bvz$ is to be inferred
from noisy multilinear  measurements. Namely, for each unordered
$k$-uple 
$\{i_1,i_2,\dots,i_k\}\subseteq [n]$, we measure  
\begin{align}
\bX_{i_1,i_2,\dots,i_k} =
\beta(\bvz)_{i_1}(\bvz)_{i_2}\cdots(\bvz)_{i_k} +
\bZ_{i_1,i_2,\dots,i_k} \, ,
\end{align}
with $\bZ$ Gaussian noise (see below for a precise definition)
and wish to reconstruct $\bvz$. 
In tensor notation, the observation model reads (see the end of this
section for notations)
\begin{align}\tag*{Spiked Tensor Model}
\bX = \beta\, \bvz^{\otimes k} + \bZ~~. \label{eq:SymmetricModel}
\end{align}
This is analogous to the so called `spiked covariance model'  used to
study matrix PCA in high dimensions
\cite{johnstone2009consistency}.

It is immediate to see that  maximum-likelihood estimator  $\bv^{\sML}$ is given by a
solution of the following problem 
\begin{align}\tag*{Tensor PCA}
\text{maximize} & \;\;\;\, \<\bX,\bv\ok\>  ,\label{eq:TPCA}\\
\text{subject to} & \;\;\;\; \|\bv\|_2= 1\,  .\nonumber
\end{align}
%
Solving it exactly is --in general-- NP hard \cite{hillar2013most}.

We next summarize our results. 
Note that, given a completely observed rank-one symmetric tensor
$\bvz\ok$ (i.e. for $\beta=\infty$), 
it is easy to recover the vector $\bvz\in\reals^n$.  It is therefore
natural to ask the question
\emph{for which  signal-to-noise ratios one can one still reliably
  estimate $\bvz$?} 
The answer appears to depend dramatically on the computational
resources\footnote{Here we write $F(n)\lesssim G(n)$ if there
  exists a constant $c$ independent of $n$ (but possibly dependent on
  $k$), such that $F(n)\le c\, G(n)$}.
\begin{description}
\item[Ideal estimation.] Assuming unbounded computational
  resources, we can solve the  \ref{eq:TPCA} optimization problem and
  hence implement the maximum likelihood estimator $\hbv^{\sML}$.
We use recent results in probability theory to show that this approach
is successful for $\beta\ge \mu_k$ (here $\mu_k$ is a constant given
explicitly below, with $\mu_k = \sqrt{k\log k}(1+o_k(1))$). In particular, above
this threshold\footnote{Note that, for $k$ even,  $\bvz$
  can only be recovered modulo sign. For the sake of simplicity, we
  assume here that this ambiguity is correctly resolved.} we have, with high probability,
\begin{align}
\|\hbv^{\sML}-\bvz\|_2^2 \le \frac{2.01\, \mu_k}{\beta}\, .
\end{align}
We use an information-theoretic argument to show that no approach can
do significantly better, namely no procedure can estimate $\bvz$
accurately for $\beta\le c \sqrt{k}$ (for $c$ a universal constant).
\item[Tractable estimators: Unfolding.] We consider two approaches to estimate
  $\bvz$ that can be implemented in polynomial time.  The first approach is
  based on tensor unfolding: starting from the tensor
  $\bX\in\bigotimes^k\reals^n$, we produce a matrix $\M(\bX)$ of
  dimensions $n^q\times n^{k-q}$. We then perform matrix PCA on $\M(\bX)$. We show that this
  method is successful for $\beta\gtrsim n^{(\lceil
    k/2\rceil-1)/2}$ (provided we choose $q= \lceil k/2\rceil$).

A heuristics argument suggests that the necessary and sufficient condition for tensor
unfolding to succeed is indeed  $\beta\gtrsim n^{(k-2)/4}$ (which is
below the rigorous bound by a factor $n^{1/4}$ for $k$ odd). We can indeed
confirm this conjecture for $k$ even and under an asymmetric noise
model.
Numerical simulations confirm the conjecture for $k=3$.

\item[Tractable estimators: Power iteration.]
We then consider a simple tensor power iteration method, that proceeds
by repeatedly applying the tensor to a vector. 
We prove that, initializing this iteration uniformly at random, it
converges very rapidly to an accurate estimate provided $\beta\gtrsim
n^{(k-1)/2}$.
A heuristic argument suggests that the correct necessary and
sufficient threshold is given by $\beta\gtrsim
n^{(k-2)/2}$.
In other words, power iteration is substantially less powerful than 
unfolding.

\item[Tractable estimators: Warm-start power iteration.]
 Motivated by the last observation, we consider a `warm-start'
power iteration algorithm, in which we initialize power iteration with
the output of tensor unfolding. 
This  approach appears to have the same  threshold signal-to-noise
ratio as simple unfolding, but significantly better accuracy above
that threshold.

We also study a number of variations on this, with improved unfolding methods.

\item[Tractable estimators: Approximate Message Passing.] Finally we
consider an approximate message passing (AMP) algorithm \cite{DMM09,BM-MPCS-2011}. Such
algorithms proved effective in compressed sensing and several other
estimation problems. We show that the behavior of AMP is
qualitatively similar to the one of naive power iteration.
In particular,  AMP fails for any $\beta$ bounded as
$n\to\infty$.
\item[Side information.] Given the above computational complexity barrier, it is natural to
  study weaker version of the original problem. Here we assume that
  extra information about $\bvz$ is available. This can be provided by
  additional measurements or by approximately solving a related
  problem, for instance a matrix PCA problem as in
  \cite{anandkumar12}. We model this additional information as $\by =
  \gamma \bvz + \bg$ (with $\bg$ an independent Gaussian noise
  vector), and incorporate it
  in the initial condition of AMP algorithm. 
We characterize exactly  the
threshold value $\gamma_*=\gamma_*(\beta)$ above which AMP 
converges to an accurate estimator.
\end{description}

The thresholds for various classes of algorithms are summarized below.

\vspace{0.25cm}

\begin{tabular}{|c|c|c|}
\hline
 Method & Required  $\beta$  (rigorous) & Required  $\beta$  (heuristic) \\ 
 \hline  
 Tensor Unfolding & $O(n^{(\lceil k/2\rceil-1)/2})$ & $n^{(k-2)/4}$\\
 Tensor Power Iteration (with random init.) &  $O(n^{(k-1)/2})$ &  $n^{(k-2)/2}$\\
 Maximum Likelihood & $1$ & --\\
 Information-theory lower bound & $1$ & --\\
 \hline
\end{tabular}

\vspace{0.25cm}

We will conclude the paper with some insights that we believe 
provide useful guidance for tensor factorization heuristics. We
illustrate these insights through simulations.

Throughout the paper, proofs will be  deferred to the Appendices.

\subsection{Notations} 

We will use lower-case boldface for vectors (e.g. $\bu$, $\bv$, and so
on) and upper-case boldface for matrices and tensors (e.g. $\bX,\bZ$,
and so on). 
The ordinary scalar product and $\ell_p$ norm over vectors are denoted
by $\<\bu,\bv\> = \sum_{i=1}^n\bu_i\bv_i$, and $\|\bv\|_p$. 
We write
$\bbS^{n-1}$ for the unit sphere in $n$ dimensions
\begin{align}
\bbS^{n-1} \equiv\big\{\bx\in\reals^n:\; \;\|\bx\|_2=1\big\}\, .
\end{align}

Given $\bX \in \bigotimes^k \reals^n$ a real $k$-th order tensor,
we let $\{\bX_{i_1,\dots,i_k}\}_{i_1,\dots,i_k}$ denote its coordinates and define a map $\bX:\reals^n\to
\reals^n$, by letting, for $\bv\in \reals^n$,
\begin{equation}\label{eq:tensorVectorProduct}
\bX \{ \bv \}_i=   \sum_{j_1, \cdots , j_{k-1} \in [n]}
\bX_{i,j_1, \cdots , j_{k-1}} \ \bv_{j_1}\cdots \bv_{ j_{k-1}}\, .
\end{equation}
The outer product of two tensors is $\bX\otimes \bY$, and, for
$\bv \in \reals^n$, we define $\bv^{\otimes k} = \bv \otimes \cdots
\otimes \bv \in \bigotimes^k\reals^n$ 
as the $k$-th outer power of $\bv$. We define the inner product of two tensors $\bX, \bY \in \bigotimes^k \reals^n$ as 
\begin{align}
 \< \bX , \bY \> = \sum_{i_1, \cdots ,i_k \in [n] } \bX_{i_1, \cdots
  ,i_k} \bY_{i_1, \cdots ,i_k}~~.
\end{align}
We define the Frobenius (Euclidean) norm of a tensor $\bX$  by
$\|\bX\|_F = \sqrt{\<\bX,\bX\>}$, and its operator 
norm by 
\begin{align}
\|\bX\|_{op} \equiv\max \{ \< \bX,\bu_1\otimes \cdots \otimes \bu_k
\>~:~\forall i\in [k]~,~\|\bu_i\|_2\leq 1 \}.
\end{align}
It is easy to check that this is indeed a norm. For the special case
$k=2$, it reduces to the ordinary $\ell_2$ matrix operator norm
(equivalently, to the largest singular value of $\bX$).

For a permutation $\pi \in \mathfrak S_k$, we will denote by $\bX^\pi$ the tensor
with permuted indices $\bX^\pi_{i_1, \cdots, i_k} = \bX_{\pi(i_1),
  \cdots , \pi(i_k)}$. We call the tensor $\bX$ \emph{symmetric} if,
for any permutation $\pi \in \mathfrak S_k$, $\bX^\pi = \bX$. 
It is proved \cite{Waterhouse90} that, for symmetric tensors, 
the value of problem \ref{eq:TPCA} coincides with $\|\bX\|_{op}$ up to
a sign.
More precisely, for symmetric tensors we have the equivalent
representation
\begin{align}
\|\bX\|_{op} \equiv\max \{ |\< \bX,\bu\ok\>|~:~~~\|\bu\|_2\leq 1 \}.
\end{align}

We denote by  $\bG \in \bigotimes^k \reals^n$  a tensor with
independent and identically distributed entries 
$\bG_{i_1, \cdots, i_k}\sim\normal(0,1)$ (note that this tensor is not symmetric).
We define the {\em symmetric standard normal} noise tensor $\bZ \in \bigotimes^k \reals^n$
by
\begin{align}\label{eq:symNoiseDefinition}
\bZ =\frac{1}{k!  }\sqrt{\frac k n} \sum_{\pi \in \mathfrak S_k }
\bG^\pi\, .
\end{align}
Note  that the subset of entries with unequal indices form an i.i.d.
collection $\{\bZ_{i_1,i_2,\dots,i_k} \}_{i_1<\dots<i_k} \sim \normal(
0,1/(n(k-1)!))$.
The normalization adopted here is convenient because it yields, for
any fixed vector $\bv\in\reals^{n}$,
\begin{align}
\bX\{\bv\} = \beta\<\bvz,\bv\>^{k-1}\, \bvz +
\frac 1 {\sqrt{ n}} \|\bv\|^{k-1}_2\bg + \bo(1)\, .\label{eq:XMultiplication}
\end{align}
where $\bg\sim\normal(0,\id_{n})$, and $\bo(1)$ is a vector with
$\|\bo(1)\|_2\to 0$ in probability as $n\to\infty$. We further have, 
\begin{align}
\E\big\{\< \bZ , \bv^{\otimes k}\>^2\big\} =\frac{k}{n}\E
\big\{ \< \bG, \bv^{\otimes k} \>^2\big\} =\frac{k}{n}
\|v\|_2^{2k} \, ,
\end{align}
and 
\begin{align}\label{eq:XvvvDistribution}
\< \bX\{\bv\},\bv\>  = \beta\<\bvz,\bv\>^k+\sqrt{\frac k n} g\, ,
\end{align}
with $g\sim\normal(0,1)$. 
Finally notice that, for $k$ even, in  \ref{eq:SymmetricModel}, the
vector $\bvz$ can always be recovered up to a sign flip. This suggest
the use of the loss function
\begin{align}
\Loss(\hbv,\bvz) \equiv
\min\Big(\|\hbv-\bvz\|^2_2,\|\hbv+\bvz\|^2_2\Big) =
2-2|\<\hbv,\bvz\>|\, .
\end{align}
%

\section{Ideal estimation}\label{sec:infoTheoryandUpperBounds}

In this section we consider the problem of estimating $\bvz$ under the
 \ref{eq:SymmetricModel}, when no constraint is
imposed on the complexity of the estimator. Our first result is a
lower bound on the loss of \emph{any} estimator. 
\begin{theorem}\label{th:infoTheoretic}
For any estimator $\hbv=\hbv(\bX)$  of $\bvz$ from data
$\bX$, such that $\|\hbv(\bX)\|_2=1$ (i.e. $\hbv : \otimes^k \reals^n \to
    \bbS^{n-1}$), we have, for all $n\ge 4$,
\begin{align}
\beta\le \sqrt{\frac{k}{10}}\;\;\Rightarrow\;\;
\E\,\Loss(\hbv,\bvz)\ge \frac{1}{32}\, .
\end{align}
\end{theorem}
%
%
In order to establish a matching upper bound on the loss, we consider the
maximum likelihood estimator $\hbv^{\sML}$, obtained by solving the \ref{eq:TPCA}
problem.
As in the case of matrix denoising, we expect the properties of
this estimator to depend on signal to noise ratio $\beta$, and on the
`norm' of the noise
$\|\bZ\|_{op}$ (i.e. on the value of the optimization problem
\ref{eq:TPCA} in the case $\beta=0$). 
For the matrix case $k=2$, this coincides with the largest eigenvalue of
$\bZ$. Classical  random matrix theory shows that --in this case--
$\|\bZ\|_{op}$ concentrates tightly around $2$
\cite{Geman,Szarek:survey,BaiSilverstein}.

It turns out that tight results for $k\ge 3$ 
follow immediately from a technically sophisticated
 analysis of  the stationary points of random
Morse functions by Auffinger, Ben Arous and Cerny  \cite{Auffinger13}.
(See Appendix \ref{sec:BenArous} for further background.)
\begin{lemma}\label{lem:BenArous}
There exists a sequence of real numbers $\{ \mu_k\}_{k\geq 2}$,  such that 
 \begin{align}
\label{eq:muKupperBound}
\lim\sup_{n \to \infty}\|\bZ\|_{op}   &\le  \mu_k \hspace{1cm}
\mbox{($k$ odd)},\\
\lim_{n \to \infty}\|\bZ\|_{op}   &=  \mu_k \hspace{1cm}
\mbox{($k$ even)}.
\end{align}
Further $\|\bZ\|_{op} $ concentrates tightly around its
expectation. Namely, for any $n, k$ 
\begin{align}
\prob\big(\big|\|\bZ\|_{op}  -\E\|\bZ\|_{op}  \big|\ge s\big)\le 2\,
e^{-ns^2/(2k)}\, .
\end{align}
Finally $\mu_k = \sqrt{k \log k}(1+o_k(1))$ for large $k$.
\end{lemma}
An explicit expression for the quantity $\mu_k$ is given in Appendix 
\ref{sec:operatorNorm} (which also contains a proof, that uses \cite{Auffinger13}).  Evaluating this expression for small values
of $k$, we get  the following explicit values, that we also compare with
the large-$k$ asymptotics $\sqrt{k\log k}$.
(It is not hard to increase the number of digits in these evaluations,
using the expressions in Appendix.)

\vspace{0.5cm}

\begin{center}
\begin{tabular}{l|l|l}
$k$ & $\mu_k$ & $\sqrt{k\log k}$\\
\hline
$3$ & $2.8700$ & $1.8154$ \\
$4$ & $3.5882$ & $2.3548$\\
$5$ & $4.2217$ & $2.8368$\\
$10$ &$6.7527$ & $4.7985$\\
$100$&$27.311$ & $21.460$
\end{tabular}
\end{center}

\vspace{0.5cm}

For instance, this table indicates that a large order-$3$ Gaussian
tensor should have  $\|\bZ\|_{op}  \approx 2.87$, while a large
order $10$ tensor has $ \|\bZ\|_{op}   \approx 6.75$.
As a simple consequence of Lemma \ref{lem:BenArous}, we establish an
upper bound on the error incurred by the maximum likelihood estimator,
see Section \ref{sec:ProofML} for a proof.
\begin{theorem}\label{th:nonconvexOpt}
Let  $\mu_k$ be the sequence of real numbers introduced above. 
Letting $\hbv^{\sML}$  denote  the maximum likelihood
estimator (i.e. the solution of  \ref{eq:TPCA}), 
we have for $n$ large enough, and all $s>0$
\begin{align}
\beta\ge \mu_k \Rightarrow \Loss(\hbv^{\sML},\bvz)\le \frac{2}{\beta}  \left ( \mu_k +   s \right )\, ,
\end{align}
with probability at least $1- 2e^{-ns^2/(16 k)}$.
\end{theorem}

%
%
The following upper bound on the value of the problem \ref{eq:TPCA}
is proved using Sudakov-Fernique inequality. While it is looser than
Lemma \ref{lem:BenArous} (corresponding to the case
$\beta=0$), we expect it to become sharp for $\beta\ge \beta_k$ a
suitably large constant. We refer to Appendix
\ref{sec:ProofUpperBoundX} for its proof. 
\begin{lemma}\label{lem:upperBoundX}
Under \ref{eq:SymmetricModel} model, we have
\begin{equation}
\lim \sup_{n \to \infty} \E \|\bX\|_{op}   \leq \max_{\tau\ge 0} \Big\{\beta \left ( \frac {\tau}{\sqrt{1+\tau^2}} \right )^k +  \frac k{ \sqrt{1+\tau^2}}\Big\}~~.
\end{equation}
Further, for any $s\ge 0$,
\begin{align}
\prob\big(\big|\|\bX\|_{op}  -  \E \|\bX\|_{op}  \big|\ge s\big)\le 2\,
e^{-ns^2/(2k)}\, .
\end{align}
\end{lemma}

\subsection{Historical Background}

At this point it is useful to pause,  in order to provide some further background.
The random cost function $\bv\mapsto H_{\bZ}(\bv) \equiv
\<\bZ,\bv\ok\>$ (defined on the unit sphere $\bv\in \bbS^{n-1}$) was studied in the context of statistical physics
under the name of `spherical p-spin model.'
In particular, Crisanti and S\"ommers \cite{crisanti1992sphericalp} used the non-rigorous
replica method from spin glass theory to compute the asymptotic value
$\mu_k$.  Their results were confirmed rigorously by Talagrand
\cite{talagrand2006free}.

The most striking prediction from statistical physics is that the
function $H_{\bZ}(\bv)$  has an exponential number of local maxima 
on the unit sphere \cite{crisanti1995thouless}. Furthermore, there exists $\eta_k<\mu_k$ such
that, for each $x\in [\eta_k,\mu_k)$ the number of local maxima with
value $H_{\bZ}(\bv) \approx x$ is $\exp\{\Theta(n)\}$.
In \cite{Auffinger13} rigorous evidence is developed to this support picture.

In the \ref{eq:SymmetricModel} these local maxima translate into {\em undesired} local maxima of
$\<\bX,\bv\ok\>$.  It is natural to guess that  these local maxima
affect local iterative algorithms, and that these
do not converge to a good estimate of $\bvz$ unless they are
initialized within thee `basin of attraction' of $\bvz$.
The analysis in the next sections confirms this intuition.

\section{Tensor Unfolding}
\label{sec:Unfolding}

A  simple and popular heuristics to obtain tractable
estimators of $\bvz$ consists in constructing a suitable matrix with
the entries of $\bX$, and performing principal component analysis on this
matrix. 
Since the number of distinct entries of $\bX$ is of order
$n^k$,
the resulting matrix $\M_q(\bX)$ has dimension $\Theta(n^q)\times
\Theta(n^{k-q})$. This operation is variously referred as
matricization, unfolding, flattening.
While the details of this construction can vary, we
do not expect them to affect qualitatively our results, that we
summarize for the sake of convenience:
\begin{enumerate}
\item The best way to unfold $\bX$ amounts to form a matrix as
  square as possible.
\item Setting $b = (\lceil k/2\rceil - 1)/2$ (in particular $b =
  1/2$ for $k\in\{3,4\}$), the unfolding approach succeeds when
  $\beta$ is larger than $n^b$. This is to be compared with
  $\beta = \Theta(1)$ that is sufficient for the maximum likelihood
  estimator (see previous section).

Based on heuristic arguments, we believe that the tight threshold is
$\beta\gtrsim n^{(k-2)/4}$ (i.e. that $\beta\gtrsim n^{(k-2)/4}$ is
both necessary and sufficient --modulo constants).
\item A sharper analysis is possible when the symmetric noise tensor
  $\bZ$ in our \ref{eq:SymmetricModel} is replaced by non-symmetric
  Gaussian noise, and $k$ is even. In particular, we can confirm the 
above conjecture in this case. (As mentioned, we expect similar results to hold
  more generally.)

In this case, if $\beta\le (1-\eps) n^b$, then the estimator from
unfolding is essentially orthogonal to the signal $\bvz$. On the
other hand, if  $\beta\ge (1+\eps) n^b$, we construct an estimator
with $|\<\hbv,\bvz\>|\to 1$. 
\item We achieves the remarkable behavior at
  the last point by a \emph{recursive unfolding} method. In a
  nutshell we perform principal component analysis on $\M_q(\bX)$,
construct a matrix out of the principal vector, and then perform again
principal component analysis.
\end{enumerate}

\subsection{Symmetric noise}

For an integer  $0\le q \le k$, we  introduce  the unfolding (also referred to as matricization or {\tt reshape}) 
operator $\M_q : \otimes^k\reals^n \to \reals^{n^q \times n^{k-q}}$ as
follows.
For  any indices $ i_1,i_2,\cdots, i_k\in [n]$,
we let  $a = 1+\sum_{j=1}^q (i_j-1) n^{j-1}$ and $b = 1+\sum_{j=q+1}^k
(i_j-1) n^{j-q-1}$, and define
\begin{equation}\label{eq:defMatricization} 
\left [ \M_q(\bX)\right ]_{a,b}= \bX_{i_1,i_2,\cdots,i_k}~~.~~
\end{equation}
Standard convex relaxations of low-rank tensor estimation problem
compute factorizations of   $\M_q(\bX)$\cite{Tomika11,Liu12,mu13,Romera13}. 
Not all unfoldings (choices of $q$) are equivalent.
It is natural to expect that this approach will be successful only if
the signal-to-noise ratio exceeds the operator norm of the
unfolded noise $\|\M_q(\bZ)\|_{op}$. The next lemma
suggests that the latter is minimal when $\M_q(\bZ)$ is `as square
as possible' .
A similar phenomenon was observed in a different context in \cite{mu13}.
\begin{lemma}\label{lem:upperBoundMatricized} 
For any integer $0\leq q\leq k$ we have, for some universal constant $C_k$,
 \begin{align}
 \frac{1}{\sqrt{(k-1)!}}\, n^{\max(q-1,k-q-1)/2}\,\Big(1-\frac{C_k}{n^{\max(q,k-q)}}\Big)\le
\E
\|\M_q(\bZ)\|_{op}&\leq \sqrt k \left ( n^{(q-1)/2} + n^{(k-q-1)/2}
\right )~~.\label{eq:NormMatricized}
 \end{align}
 For all $n$ large enough, both  bounds are minimized for $q = \lceil
 k/2 \rceil$.
Further
\begin{align}
\prob\Big\{\big|\|\M_q(\bZ)\|_{op}-\E\|\M_q(\bZ)\|_{op} \big| \ge t
\Big\}\le 2\, e^{-nt^2/(2k)}\, . \label{eq:ConcentrationMatricized}
\end{align}
\end{lemma}
\begin{proof}
The concentration bound (\ref{eq:ConcentrationMatricized}) follows because,
for $\bu\in\reals^{n^q}, \bv\in\reals^{n^{k-q}}$ of norm $1$, the function
\begin{align}
\<\bu,\M_q(\bZ) \bv\>= \frac{1}{k!}\,\sqrt{\frac{k}{n}}
\sum_{\pi \in \mathfrak S_k} \<\bu,\M_q(\bG^{\pi}) \bv\>
\end{align}
is a Lipschitz function of the Gaussian vector $\bG$ with modulus at
most $\sqrt{k/n}$. Hence the same holds for $\|\M_q(\bZ)\|_{op} =
\max_{\bu,\bv}\<\bu,\M_q(\bZ) \bv\>$,
and the claim follows from Gaussian concentration of measure. 

For the upper bound in Eq.~(\ref{eq:NormMatricized}), note that $\M_q(\bG^\pi)$ has i.i.d. standard normal entries.
The proof follows from the definition
(\ref{eq:symNoiseDefinition}) together with triangular inequality 
and standard bounds on  the norm of the random Gaussian matrices 
$\M_q(\bG^\pi) \in  \reals^{n^q \times n^{k-q}}$ \cite{Szarek:survey}: 
\begin{align*}
\E \|\M_q(\bZ)\|_{op}= &  \frac{1}{k!} \sqrt {\frac{k}{n}} \E \|\M_q( \sum_\pi \bG^\pi)\|_{op} \\
\leq & \frac{1}{k!} \sqrt {\frac{k}{n}} \sum_\pi   \E  \|\M_q(  \bG^\pi)\|_{op} \\
\leq & \sum_\pi   \frac{1}{k!} \sqrt {\frac{k}{n}} \left ( n^{q/2} + n^{(k-q)/2} \right )\\
 = & \sqrt k \left ( n^{(q-1)/2} + n^{(k-q-1)/2}  \right )~~.
\end{align*}

For the upper bound in Eq.~(\ref{eq:NormMatricized}), note that
\begin{align}
\min(n^{q},n^{k-q}) \E\{\|\M_q(\bZ)\|_{op}^2\} \ge \E \|\M_q(\bZ)\|_{F}^2 =
\frac{1}{k!}\,\frac{k}{n}
\sum_{\pi \in \mathfrak S_k}\E\{\<\bG,\bG^{\pi}\>\}\ge \frac{k}{n}\,
\frac{n^k}{k!}\, ,
\end{align}
where the last inequality is proved by considering $\pi = {\rm id}$
the identity permutation (all the terms in the sum are 
positive). The desired lower bound follows since the concentration
inequality (\ref{eq:ConcentrationMatricized}) implies 
$\E\{\|\M_q(\bZ)\|_{op}^2\}-\E\{\|\M_q(\bZ)\|_{op}\}^2\le (2k/n)$, and
we therefore have 
\begin{align}
\E\{\|\M_q(\bZ)\|_{op}\}\ge  \E\{\|\M_q(\bZ)\|_{op}^2\}^{1/2}\Big(1-
\frac{k}{n\E\{\|\M_q(\bZ)\|_{op}^2\}}\Big) \, .
\end{align}
\end{proof}
The last lemma suggests the choice $q=\lceil k/2\rceil$, which we
shall adopt in the following, unless stated otherwise. We will drop
the subscript from $\M$.

Let us recall the following standard result derived directly from
Wedin perturbation Theorem \cite{wedin72}, and stated in the context
of the spiked 
model.
\begin{theorem}[Wedin perturbation]\label{th:wedin}
Let $\bM = \beta \buz \bwz\trans + \bE \in \reals^{m \times p}$ be a
matrix with $\|\buz\|_2 = \|\bwz\|_2 = 1$. Let $\hbw$ denote the right
singular vector of $\bM$. If $\beta> 2 \|\bE\|_{op}$, then 
\begin{equation}
\Loss(\hbw,\bwz) \leq  \frac{8\|\bE\|_{op}^2}{\beta^2}~~. \label{eq:WedinUB}
\end{equation}
\end{theorem}
\begin{proof}[Proof of Theorem \ref{th:wedin}]
Note $\beta>0$ is the only singular value of $ \beta \buz \bwz\trans
$,  while the second singular value of $(\beta \buz \bwz\trans + \bE)$
is at most $\|\bE\|_{op}$.
Wedin Theorem states that,
for all $\beta>\|\bE\|_{op}$, we have
\begin{align}
 |\sin (\hbw,\bwz)| \leq \frac{\|\bE\|_{op}}{\beta-\|\bE\|_{op}}~~.
\end{align}
In particular  $|\sin (\hbw,\bwz)|\le 2\|\bE\|_{op}/\beta$ for
$\beta\ge 2\|\bE\|_{op}$. Hence the claim  (\ref{eq:WedinUB}) follows from
\begin{align}
|\<\hbw,\bwz\>| = 
|\cos (\hbw,\bwz)| \geq \sqrt{1-
  \frac{4\|\bE\|_{op}^2}{\beta^2}} \geq 1- 
\frac{4\|\bE\|_{op}^2}{\beta^2}~~.
\end{align}
\end{proof}
\begin{theorem}\label{th:unfold}
Letting $\bw=\bw(\bX)$  denote the top right singular vector of
$\M(\bX)$, we have the following, for some universal constant $C=C_k>0$,
and $b \equiv (1/2)(\lceil k/2\rceil -1)$. 

If $\beta\ge 5\, k^{1/2}\, n^{b}$ then,
with probability at least $1-n^{-2}$, we have
\begin{align}
\Loss\Big(\bw,\vec \big( \bvz^{\otimes \lfloor k/2\rfloor }
\big)\Big)\le \frac{  C \, k n^{2b}}{\beta^2}\, . \label{eq:LossUB}
\end{align}
\end{theorem}
\begin{proof}[Proof of Theorem \ref{th:unfold}]
By definition we have
\begin{align}
\M(\bX) = \beta\buz\bwz\trans + \M(\bZ)\, ,
\end{align}
where $\buz = \vec ( \bvz^{\otimes \lfloor k/2\rfloor })$, $\bwz = \vec ( \bvz^{\otimes \lceil k/2\rceil })$.
We know by Lemma \ref{lem:upperBoundMatricized} that $
\|\M(\bZ)\|_{op} \leq (5/2)\sqrt{k} \, n^b$  with the claimed probability. The
loss  upper bound (\ref{eq:LossUB}) follows immediately from this
upper bound and Wedin's theorem Eq.~(\ref{eq:WedinUB}).
\end{proof}

\subsection{Asymmetric noise and recursive unfolding}

A technical complication in analyzing the random matrix $\M_q(\bX)$
lies in the fact that  its entries are not independent, because the
noise tensor $\bZ$ is assumed to be symmetric. 
In the next theorem we consider the case of 
non-symmetric noise and even $k$. This allows us to leverage upon
known results in random matrix theory \cite{paul2007asymptotics,feral2009largest,benaych2012singular} to obtain: $(i)$ Asymptotically
sharp estimates on the critical signal-to-noise ratio; $(ii)$ A lower
bound on the loss below the critical signal-to-noise ratio.
 Namely, we consider observations
\begin{align}
\btX  =\beta \bvz\ok + \frac{1}{\sqrt{n}}\bG\, . \label{eq:NonSymm}
\end{align}
where $\bG \in \otimes^k\reals^n$ is a standard Gaussian tensor
(i.e. a tensor with i.i.d. standard normal entries).

 Let $\bw=\bw(\btX)\in\reals^{n^{k/2}}$  denote the top right singular vector of
$\M(\bX)$. For $k\ge 4$ even, and define $b \equiv (k-2)/4$, as above.
By \cite[Theorem 4]{paul2007asymptotics},  or \cite[Theorem
2.3]{benaych2012singular}, we
have the following almost sure limits
\begin{align}
\beta\le  (1-\eps)n^{b} & \Rightarrow \;\;\;\;\;\;
\lim_{n\to\infty}\<\bw(\btX),\vec(\bvz^{\otimes  (k/2)})\> = 0\, ,\label{eq:TightPhTr1}\\
\beta\ge  (1+\eps)n^{b} & \Rightarrow \;\;\;\;\;\;
\lim\inf_{n\to\infty}\big|\<\bw(\btX),\vec(\bvz^{\otimes  (
k/2)})\>\big| \ge  \sqrt{\frac{\eps}{1+\eps}}\, .\label{eq:TightPhTr2}
\end{align}
In other words $\bw(\btX)$ is a good estimate of $\bvz^{\otimes
  (k/2)}$ if and only if $\beta$ is larger than $n^b$. 

We can use $\bw(\btX)\in\reals^{2b+1}$ to estimate $\bvz$ as follows. 
Construct the matricization $\M_1(\bw)\in\reals^{n\times n^{2b}}$
(slight abuse of notation) of
$\bw$ by letting, for $i\in [n]$, and $j\in [n^{2b}]$,
\begin{align}
\M_1(\bw)_{i,j} = \bw_{i+(j-1)n}\, ,\label{eq:M1w}
\end{align}
we then let $\hbv$ to be the left principal vector of $\M_1(\bX)$.
We refer to this algorithm\footnote{In practice int might be more
  effective to use a balanced matricization at the second step. For
  instance if $k$ is a power of two one could construct a square
  matricization and repeat the same process. For analysis purposes, we
prefer the version described here.} as to \emph{recursive unfolding}.
\begin{theorem}\label{th:unfoldNonsymm}
Let $\btX$ be distributed according to the non-symmetric model
(\ref{eq:NonSymm}) with $k\ge 4$ even, define $b \equiv (k-2)/4$. 
and let $\hbv$ be the estimate obtained by two-steps recursive
unfolding. 

If $\beta\ge (1+\eps) n^{b}$ then, almost surely
\begin{align}
\lim_{n\to\infty}\Loss(\hbv,\bvz)  =0\, . \label{eq:LossUBNons}
\end{align}
\end{theorem}
\begin{proof}[Proof of Theorem \ref{th:unfoldNonsymm}]
For the sake of simplicity, we assume $\beta/n^b\to \eps$. The limit
along other sequences follows from a standard subsequence argument.

It follows from the invariance of the noise distribution in
Eq.~(\ref{eq:NonSymm}) that 
\begin{align}
\bw(\btX) = \rho_n\, \vec(\bvz^{\otimes (k/2)}) + \frac{\orho_n}{n^{k/4}}\, \bg\,
,
\end{align}
where $\bg\sim\normal(0,\id_{n^{k/2}})$. It follows from
Eq.~(\ref{eq:TightPhTr2}),
together with the almost sure limits
$\lim_{n\to\infty}\|\bg\|_2/n^{k/4}=1$ and 
$\lim_{n\to\infty}\<\bg,\vec(\bvz^{\otimes (k/2)})\> /n^{k/4}=1$ that
(almost surely)
\begin{align}
\lim_{n\to\infty} \rho_n &= \sqrt{\frac{\eps}{1+\eps}}\, ,\\
\lim_{n\to\infty} \orho_n &= \sqrt{\frac{1}{1+\eps}}\, .
\end{align}
Using the definition (\ref{eq:M1w}), we then have (recall that $b = (k-2)/4$)
\begin{align}
\M_1(\bw) = \rho_n \bvz\, \bu^{\sT} + \frac{\orho_n}{n^{(2b+1)/2}}\,
\bG'\, ,\label{eq:DecompositionM1}
\end{align}
where $\bu = \vec(\bvz^{2b})$ and $\bG'
\in \reals^{n\times n^{2b}}$ is a matrix with i.i.d standard normal
entries.
Using \cite{davidson2001local}, we have, with probability $1-e^{-\Theta(n)}$,
\begin{align}
\frac{\orho_n}{n^{(2b+1)/2}}\, \|\bG'\|_{op}\le
\frac{1}{n^{(2b+1)/2}}\, \big(n^{1/2}+n^b\big) \le \frac{2}{\sqrt{n}}\, .
\end{align}
Since $\rho_n$ is bounded away from zero as $n\to\infty$, Wedin's
theorem implies $\lim_{n\to\infty}|\< \hbv,\bvz\>| =1$, and therefore
the claim (\ref{eq:LossUBNons}).
 \end{proof}

We conjecture that the weaker condition $n\gtrsim n^{(k-2)/4}$ is
indeed sufficient also for our original symmetric noise model model,
both for $k$ even and for $k$ odd.


\section{Power Iteration}

Iterating over (multi-) linear maps induced by a (tensor) matrix is a
standard method for finding leading eigenpairs, 
see \cite{kolda2011shifted} and references therein for tensor-related results.
In this section we will consider a simple power iteration, and then
its possible uses in conjunction with tensor unfolding.
Finally, we will compare our analysis with results available in the
literature.

Approximate Message Passing (AMP) provides a different iterative strategy
and will be discussed in Section  \ref{sec:AMP}. While the qualitative
behavior is the same as for naive power iteration, a sharper
asymptotic analysis is possible for AMP.

\subsection{Naive power iteration}

The simplest iterative approach  is defined by the following recursion
\begin{align}\tag*{Power Iteration}\label{eq:tensorPiter}
\bv^0= \frac{\by}{\|\by\|_2}~~,\quad\text{and}\quad \bv^{t+1} =  \frac{\bX\{\bv^t\} }{ \|
\bX\{\bv^t\} \|_2}~~.
\end{align}
The following result establishes convergence criteria for this
iteration, first for generic noise $\bZ$ and then for standard normal
noise (using Lemma \ref{lem:BenArous}).
\begin{theorem}\label{lem:powerIterationPositiveResult} 
Assume 
\begin{align}
\beta&\ge 2\,e (k-1)\, \|\bZ\|_{op}\,, \label{eq:SNR_PowerIteration}\\
\frac{\<\by ,\bvz\>}{\|\by\|_2} &\ge
\left[\frac{(k-1)\|\bZ\|_{op}}{\beta}\right]^{1/(k-1)}\, ,\label{eq:InitPowerIteration}
\end{align}
Then for all $t\ge t_0(k)$, the power iteration estimator satisfies 
\begin{align}
\Loss(\bv^t,\bvz) \le \frac{2e\|\bZ\|_{op}}{\beta}\, .\label{eq:BoundPI}
\end{align} 
If $\bZ$ is a standard normal noise tensor, then conditions
(\ref{eq:SNR_PowerIteration}), (\ref{eq:SNR_PowerIteration}) are
satisfied with high probability provided 
\begin{align}
\beta &\ge 2ek\, \mu_k  = 6\sqrt{k^3\log k} \, \big(1+o_k(1)\big)\, ,\label{eq:SNR_PowerIteration_BIS}\\
\frac{\<\by ,\bvz\>}{\|\by\|_2} &\ge
\left[\frac{k\mu_k}{\beta}\right]^{1/(k-1)} = \beta^{-1/(k-1)} \, \big(1+o_k(1)\big)\, .  \label{eq:InitPowerIteration_BIS}
\end{align}
\end{theorem}
We next discuss two aspects of this result: $(i)$ The requirement of a
positive correlation between initialization and ground truth ; $(ii)$
Possible scenarios under  which the assumptions of Theorem
\ref{lem:powerIterationPositiveResult}
are satisfied.

Notice that we require a positive correlation of the initialization
$\by$ with the ground truth  $\bvz$. 
The underlying reason is that, if $\<\bv^0,\bvz\>$ is small, then
$\<\bv^t,\bvz\>$  remains small
at all subsequent iterations. In order to clarify this point, it is
instructive to 
compute the distribution of $\bv^1$ for standard Gaussian noise $\bZ$.
We let 
\begin{align}
\ttau_0 \equiv \frac{\<\bvz,\by\>}{\|\by\|_2}\, .
\end{align}
Using Eq.~(\ref{eq:XMultiplication}) and the fact that $\bv^0$ is
independent of $\bZ$, we get
\begin{align}
\bv^{1} = \frac{\beta\tau_0^{k-1}\bvz+ n^{-1/2}\bg
}{\sqrt{\beta^2\tau_0^{2(k-1)} +1}} + \bo(1)
\end{align}
where $\bg\sim\normal(0,\id_{n})$, and $\bo(1)$ is a vector with
$\|\bo(1)\|_2\to 0$ in probability as $n\to\infty$. 
In particular
\begin{align}
\ttau_1 = \<\bv^1,\bvz\> = \frac{\beta\ttau_0^{k-1}
}{\sqrt{\beta^2\ttau_0^{2(k-1)} +1}} + o(1)\, .\label{eq:FirstIterationPowerIteration}
\end{align}
In particular $\tau^1\lesssim \tau^0$ only if
$\beta\tau_0^{k-2}\gtrsim 1$,
or, equivalently, $\<\by,\bvz\>/\|\by\|_2\gtrsim \beta^{-1/(k-2)}$.
This suggest that the condition in
Eq.~(\ref{eq:InitPowerIteration_BIS}) is not too far from being tight
(in the sense that the exponent $-1/(k-1)$ can at best replaced by $-1/(k-2)$).

In general we cannot assume that an
initialization satisfying the conditions of Theorem
\ref{lem:powerIterationPositiveResult}.
Hence, unlike for ordinary matrix factorization,
\emph{power iteration is not a practical solution to the tensor principal
  component problem}. There are however circumstances under which a
sufficiently good initialization exists.
\begin{description}
\item[Extremely low noise.] If $y$ is a uniformly random vector on the
  unit sphere, then $\<\bvz,\by\>$ is approximately normal with mean
  zero and variance $1/n$. For instance $|\<\bvz,\by\>|\ge 1/\sqrt{n}$
with probability roughly $0.32$.

Comparing this with condition (\ref{eq:InitPowerIteration}), we obtain
that a random initialization succeed with positive probability if
\begin{align}
\beta\ge (2n)^{(k-1)/2}\|\bZ\|_{\rm op}\,. 
\end{align}
For standard Gaussian noise, this amounts to requiring $\beta \ge
(2n)^{(k-1)/2}\mu_k$.
The above heuristic analysis suggests that the correct condition
should be  $\beta \gtrsim n^{(k-2)/2}$.
\item[Additional side information.] Additional information might be
  available about the vector $\bvz$. This information can be used for
  initiating the power iteration. In the next section we consider the
  special case in which tensor unfolding is used for initializing
  power iteration. 
\end{description}

\subsection{Comparison with Tensor Unfolding}\label{sec:matriceWarmStart}

It is instructive to compare the result of the previous section 
with the ones for tensor unfolding, cf. Section \ref{sec:Unfolding}.
Summarizing, for standard Gaussian noise
\begin{itemize}
\item Tensor unfolding is guaranteed to succeed provided $\beta\gtrsim n^{b}$,
with $b = (\lceil k/2\rceil -1)/2$.  We conjecture that a necessary
and sufficient condition is in fact $\beta\gtrsim n^{(k-2)/4}$
(e.g. $\beta\gtrsim n^{1/4}$ for order $3$ tensors).
\item Power iteration, with random initialization requires
  $\beta\gtrsim n^{(k-1)/2}$.   Our heuristic calculation suggests
  that a necessary and sufficient condition is in fact $\beta\gtrsim
  n^{(k-2)/2}$
(e.g. $\beta\gtrsim n^{1/2}$ for order $3$ tensors)..
\end{itemize}

In other words, tensor unfolding is successful 
under a signal-to-noise ratio that is order of magnitudes smaller than
power iteration.
This suggests the following \emph{warm start} procedure:
$(i)$ Compute a first estimate $\hbv^{\rm Unfold}$ of $\bvz$ using tensor
unfolding; $(ii)$ Use this as initialization for the power iteration,
hence setting $\bv^0= \hbv^{\rm Unfold}$. We will explore
this approach numerically in Section \ref{sec:num}.

\subsection{Related work}

As mentioned above, power iteration is a natural approach to tensor factorization
and was studied in several earlier papers.
Most recently, interest within machine learning was spurred by 
\cite{anandkumar12}.

Our Theorem \ref{lem:powerIterationPositiveResult} is analogous
to the main result of \cite{anandkumar12} although incomparable:
\begin{itemize}
\item In \cite{anandkumar12} the `signal' part of the tensor $\bX$ is
assumed to have an orthogonal decomposition $\sum_{i=1}^n \lambda_i
\bv_i^{\otimes k}$ with $\min_i(\lambda_i)$ bounded away from zero.
Here, the signal part has rank one (equivalently, all the
$\lambda_i$'s but one vanish).
\item In   \cite{anandkumar12} only the case of third order tensors
  ($k=3$) is considered. We characterize power iteration for general $k$.
\item We establish convergence in a number of iterations $t$ that is
  \emph{independent of the dimensions} $n$. In   \cite{anandkumar12}
the number of iterations is bounded by a polynomial in $n$.
\item  We evaluate our bounds in the case of Gaussian noise. This
  allows a comparison with other methods, such as tensor unfolding.
\end{itemize}

\section{Asymptotics via Approximate Message Passing }\label{sec:AMP}

Approximate message passing (AMP) algorithms \cite{DMM09,BM-MPCS-2011} proved successful in
several high-dimensional estimation problems including compressed
sensing, low rank matrix reconstruction, and phase retrieval
\cite{fletcher2011neural,kamilov2012approximate,schniter2011approximate,schniter2012compressive}.
An appealing feature of this class of algorithms is that their
high-dimensional limit can be characterized exactly through a
technique known as `state evolution.'
Here we develop an AMP algorithm for tensor data, and its state
evolution analysis focusing on the fixed $\beta$,
$n\to\infty$ limit. Proofs follows the approach of \cite{BM-MPCS-2011}
and will be presented in a journal publication.

In a nutshell, our AMP for Tensor PCA can be viewed as a sophisticated
version  of the power iteration method of the last section.
 With
the notation
$f(\bx) = \bx / \|\bx\|_2$, we define the AMP iteration over vectors
$\bv^t \in \reals^n$ by $\bv^0 = \by$, 
$f(\bv^{-1}) = 0$, and
\begin{align}\tag*{AMP}\label{eq:AMP}
\left \{  \begin{aligned}\bv^{t+1}  & = \bX \{ f(\bv^t)\}   -\ons_t\,f(\bv^{t-1})\, ,\\
\ons_t &= (k-1) \left (  \< f(\bv^t ), f(\bv^{t-1})\>\right
)^{k-2}~. \end{aligned} \right .
\end{align}
(Note that, unlike in power iteration, we normalize $\bv^t$ `before'
multiplying it by $\bX$. This choice is equivalent but yields slightly
simpler expression.)

Our main conclusion is that \emph{the behavior of AMP is qualitatively
similar to the one of power iteration.}  However, we can establish 
stronger results in two respects:
\begin{enumerate}
\item We can prove that, unless side information
is provided about the signal $\bvz$, the AMP estimates remains
essentially orthogonal to $\bvz$, for any fixed number of iterations. 
This corresponds to a converse to Theorem \ref{lem:powerIterationPositiveResult}.
\item Since state evolution is asymptotically exact, we can prove
sharp phase transition results with explicit characterization of their locations.
\end{enumerate}

We assume that the additional information takes the form of a noisy
observation 
$\by = \gamma\, \bvz + \bz$, where $\bz\sim\normal(0,\id_n/n)$.
Our next results summarizes the state evolution analysis. Its proof is
deferred to a journal publication.
%
\begin{proposition} \label{propo:AMP_STATE}
Let $k \geq 2$ be a fixed integer. Let $\{
  \bvz(n)\}_{n\geq 1}$ be a sequence of unit norm vectors
  $\bvz(n)\in\bbS^{n-1}$. Let also $\{ \bX(n)\}_{n\geq 1}$ denote a sequence of tensors
  $\bX(n) \in \otimes^k \reals^n$ generated following
  \ref{eq:SymmetricModel}. Finally, let $\bv^t$ denote the $t$-th
  iterate produced by  \ref{eq:AMP}, and consider its orthogonal decomposition
\begin{align}
\bv^t = \bv^t_{\parallel} + \bv^t_{\perp}\, ,
\end{align}
where $\bv^t_{\parallel}$ is proportional to $\bvz$, and
$\bv^t_{\perp}$ is perpendicular.  Then $\bv^t_{\perp}$ is uniformly random, conditional on its norm.
Further, almost surely 
\begin{align}
\lim_{n \to \infty} \<\bv^t,\bvz\> =\lim_{n \to \infty} \<\bv^t_{\parallel},\bvz\> & = \tau_t \, ,\\
\lim_{n \to \infty} \|\bv^t_{\perp}\|_2 &= 1 \, ,
\end{align}
where $\tau_t$ is given recursively by letting $\tau_0=\gamma$ and, for
$t\ge 0$ (we refer to this as to \emph{state evolution}):
\begin{align}\label{eq:SE}
  \tau_{t+1}^2 = \beta^2 \left (\frac{\tau_t^2}{1+\tau_t^2}\right )^{k-1}~~.
\end{align}
\end{proposition}
Note that state evolution coincides with the equation that we derived
for the first iteration of power iteration, cf.
Eq.~(\ref{eq:FirstIterationPowerIteration}) (apart from the different scaling). It is important to notice
that for subsequent iterations $t\ge 1$, state evolution (\ref{eq:SE})
does not correctly describe naive power iteration. The reason is that
$\bv^t$ depends on $\bX$, and hence the same argument does not apply.
The AMP iteration differ  from naive power iteration because of the `memory
term', $-\ons_t\,f(\bv^{t-1})$. As shown in \cite{BM-MPCS-2011}, this
memory term approximately cancels dependencies. As a consequence, the
resulting algorithm obeys state evolution.

The following result characterizes the 
minimum required additional information $\gamma$ to allow \ref{eq:AMP}
to escape from those undesired local optima. 
We will say that $\{ \bv^t\}_t$ converges almost surely to a {\em desired local optimum} if,
almost surely, 
 \begin{equation*}
\lim_{t \to \infty} \lim_{n \to \infty}  \Loss(\bv^t /
\|\bv^t\|_2,\bvz)\le \frac{6}{\beta^2}~~.
\end{equation*}

\begin{theorem}\label{thm:AMPisGood}
Consider the \ref{eq:TPCA} problem with $k \geq 3$
and $$\beta>\omega_k \equiv \sqrt{(k-1)^{k-1} / (k-2)^{k-2}} \sim
\sqrt{ek}~~.$$ 

Then \ref{eq:AMP} converges almost surely to a desired local optimum
if and only if $\gamma > \sqrt{1/\epsilon_k(\beta)-1}$ where
$\epsilon_k(\beta)$ is the largest solution of $(1-\epsilon)^{(k-2)}\epsilon =
\beta^{-2}$, 
\end{theorem}
In the special case $k = 3$,  and $\beta>2$, assuming  $\gamma >\beta(1/2 - \sqrt{1/4 - 1/\beta^2})$, \ref{eq:AMP} tends to a desired local optimum. Numerically $\beta>2.69$ is enough for \ref{eq:AMP} to achieve $\<\bvz,\hbv\>\geq 0.9$ if $\gamma > 0.45$.

As a final remark, we note that the methods of \cite{montanari2014non}
can be used to show that, under the assumptions of Theorem
\ref{thm:AMPisGood}, for $\beta>\beta_k$ a sufficiently large constant, 
AMP asymptotically solves the optimization problem \ref{eq:TPCA}.
Formally, we have, almost surely, 
\begin{align}\label{eq:upperBoundRayleight}
\lim_{t \to \infty} \lim_{n \to \infty} \Big|\<\bX,\left (\bv^{t}\right)^{\otimes k}\>-\| \bX\|_{op}\Big|=0.
\end{align}

%
\section{Numerical experiments}\label{sec:num}

Let us emphasize two practical suggestions that arise from our work:
\begin{itemize}
\item Tensor unfolding is superior to tensor power iteration under 
our spiked model. For instance, for $k=3$, we expect tensor power iteration 
to require $\beta\gtrsim n^{1/4}$ and unfolding to require
$\beta\gtrsim n^{1/2}$.
\item For smaller values of $\beta$, iterative methods (tensor power iteration or approximate message
  passing) only produce a good estimate if the initialization has a
  scalar product with the ground truth $\bvz$ that is bounded away
  from zero.
\item  As a consequence of the above,
 side information about the unknown vector $\bvz$ can
  greatly improve performances.

A special case,  we will study the behavior of warm start algorithms that first perform a singular value
  decomposition of $\M(\bX)$, and then apply an iterative method
  (tensor power iteration or approximate message passing).
\end{itemize}
In this section we will illustrate these suggestions through numerical simulations.

Section \ref{sec:ConeConstrained} describes a refinement of tensor
unfolding that provides a tighter relaxation. Section
\ref{sec:NumericalComparison} compares different algorithms. Finally,
Section \ref{sec:SideInfo} provides additional illustration of how
side information can dramatically simplify the estimation problem.

\subsection{PSD-constrained principal component}
\label{sec:ConeConstrained}

Note that, for $\bv\in\reals^n$, the outer product $\bv \otimes \bv$
(regarded as an $n\times n$ matrix) is positive semi-definite (PSD). 
Considering the case $k=3$, we have
\begin{align}
\M(\bX) = \beta\vec(\bvz\otimes\bvz)\, \bvz^{\sT} + \M(\bZ)\, .
\end{align}
This remark suggest to perform a cone-constrained principal component 
analysis of $\M(\bX)$, where the left singular vector (viewed as a
matrix) belongs to the PDS cone.
In order to write this formally, it is convenient to introduce
the operator ${\tt reshape}_{n\times n} : \reals^{n^2} \to \reals^{n
  \times n} $ that matricizes vectors as ${\tt reshape}_{n\times
  n}(\bw)_{i,j} = \bw_{n(i-1) + j}$. The PSD-cone-constrained
principal component of $\M(\bX)$, is defined by 
\begin{align}\label{eq:defCCPrincComp}
(\hbw,\hbv) \equiv \arg \max \left \{ \< \bw, \M(\bX) \bv \> ~~:~~ {\tt
    reshape}_{n\times n}(\bw)  \succeq 0~,~\|\bw\|_2\leq 1~,~\|\bv\|_2
  \leq 1 \right \}~~.
\end{align}
This optimization problem is NP hard, since it includes copositive
programming as a special case. However
\cite{deshpande14coneconstrained} provides rigorous and empirical
evidence that problems of this type can be solved efficiently by
a projected power iteration, under statistical model dor $\bX$.

 Denoting $\P : \reals^{n^2} \to \reals^{n^2}$ the orthogonal
 projector  onto the PSD cone,  we iterate the following for $t\ge 0$, using random initialization
 of $\bu^0 \in \reals^n$, 
\begin{align}\label{eq:coneConstrained}
\left \{  \begin{aligned}\bw^{t}  & = \P(\M(\bX) \bv^t)  ,\\
\bv^{t+1} &= \M(\bX)^\sT\bw^{t} / \|\M(\bX)^\sT\bw^{t} \|_2  ~. \end{aligned} \right .
\end{align}

\subsection{Comparison of different algorithms}
\label{sec:NumericalComparison}

In Fig.~\ref{fig:comparisonAMPpowerUnfold}  we compare different
algorithms on data generated
following \ref{eq:SymmetricModel} with $k=3$, and $n \in \{25,50,100,200,400,800\}$
and for a range of values of $\beta \in[2,10]$. 
The plots represent
measured values of the absolute correlation $|\<\hbv,\bvz\>|$ versus
$\beta$, averaged over $50$ samples (except for $n=800$, where we used
$8$ samples). 

\begin{figure}
\begin{center}
\includegraphics[trim=3cm 0cm 0cm 0cm, clip=true,width = 14cm]{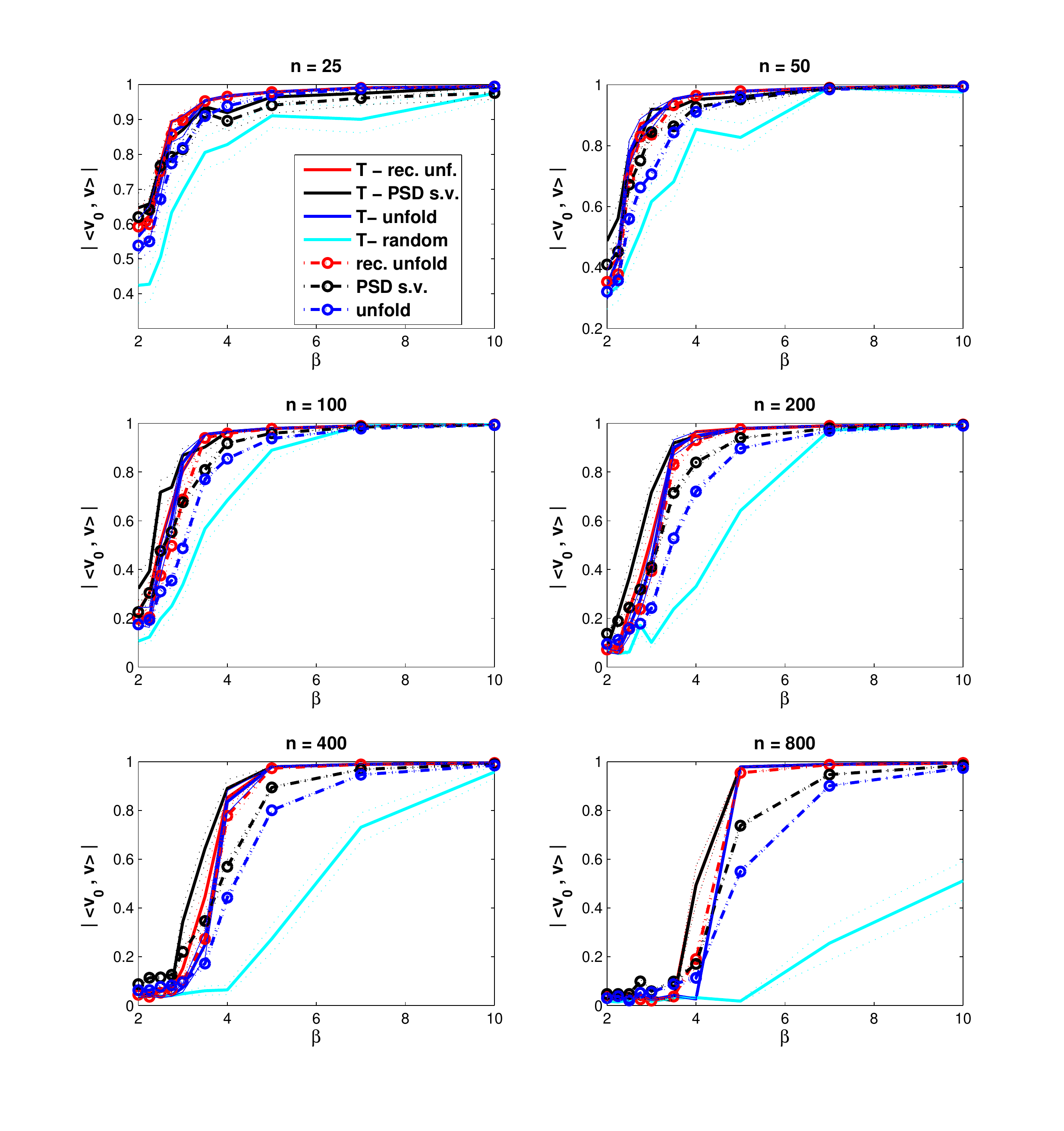}
\end{center}
\caption{Comparison of various algorithms for tensor PCA, for order
  $3$ tensors ($k=3$). Various curves correspond to different
  algorithms: {\sf unfold} (simple unfolding); {\sf rec. unfold}
  (recursive unfolding); {\sf PSD s.v.} (PSD-constrained PCA); {\sf
    T-random} (tensor power iteration with random initialization);
  {\sf T-rec.unf.}, {\sf T-PSD s.v.}; {\sf T-unfold} (tensor power
  iteration with each of the initializations above).  Light dotted lines
  are confidence bands.}\label{fig:comparisonAMPpowerUnfold}
\end{figure}
The main findings are consistent with the theory developed above:
\begin{itemize}
\item Tensor power iteration (with random initialization) performs
  poorly with respect to other approaches that use some form of tensor
  unfolding. The gap widens as the dimension $n$ increases.
\item PSD-constrained principal component analysis (described in the
  last section) is slightly superior to plain unfolding.
\item All algorithms based on initial unfolding have essentially the
  same threshold. 
Above that threshold, those that process the singular component
(either by recursive unfolding or by tensor power iteration) have
superior performances over simpler one-step algorithms.
\end{itemize}
In addition, we noted that the two iterative algorithms (\ref{eq:tensorPiter} and \ref{eq:AMP})
show very close behaviors in our experiments. 
\begin{figure}
\begin{center}
\includegraphics[trim=4cm 8cm 2cm 8cm, clip=true, width =8cm]{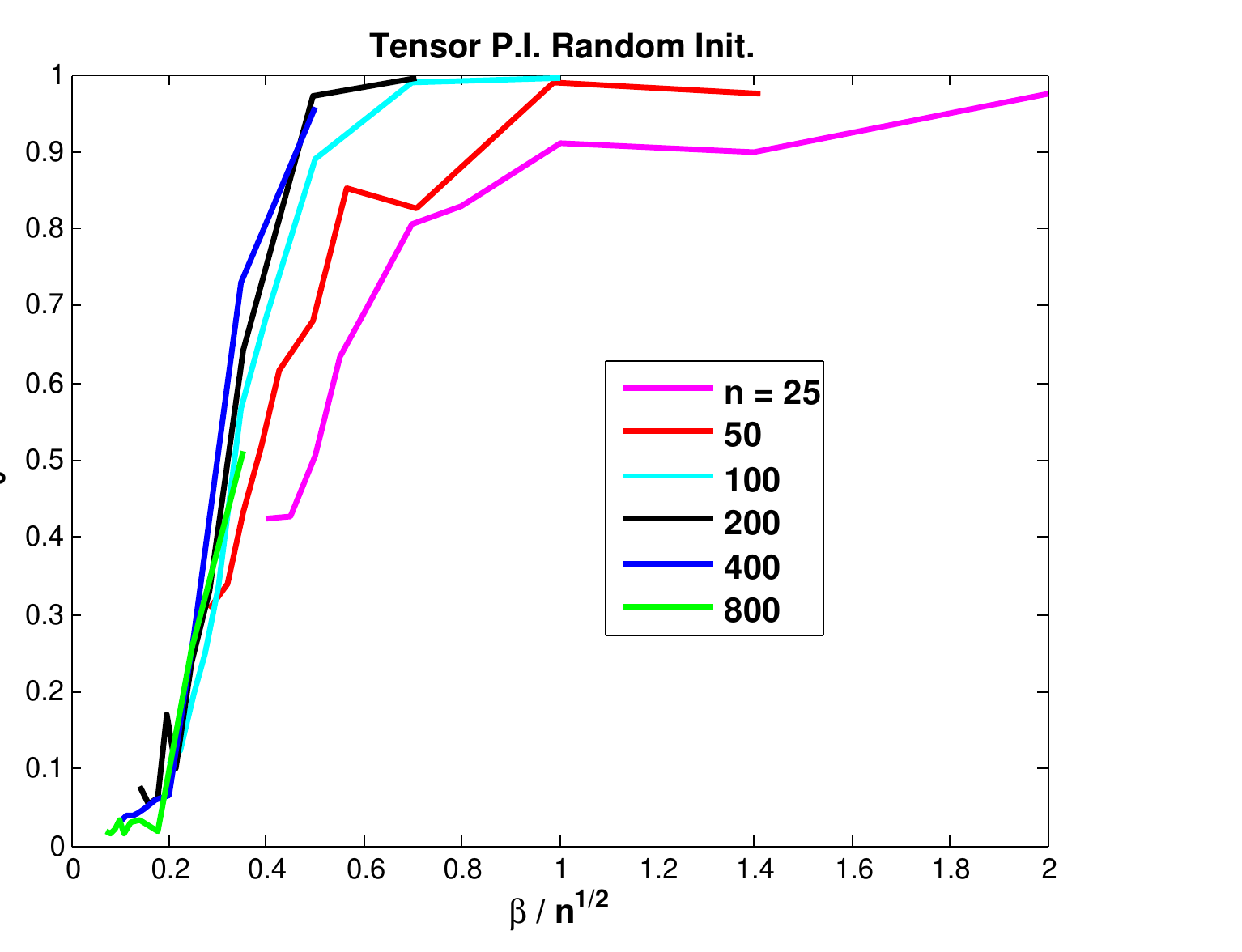}
\includegraphics[trim=4cm 8cm 2cm 8cm, clip=true, width =8cm]{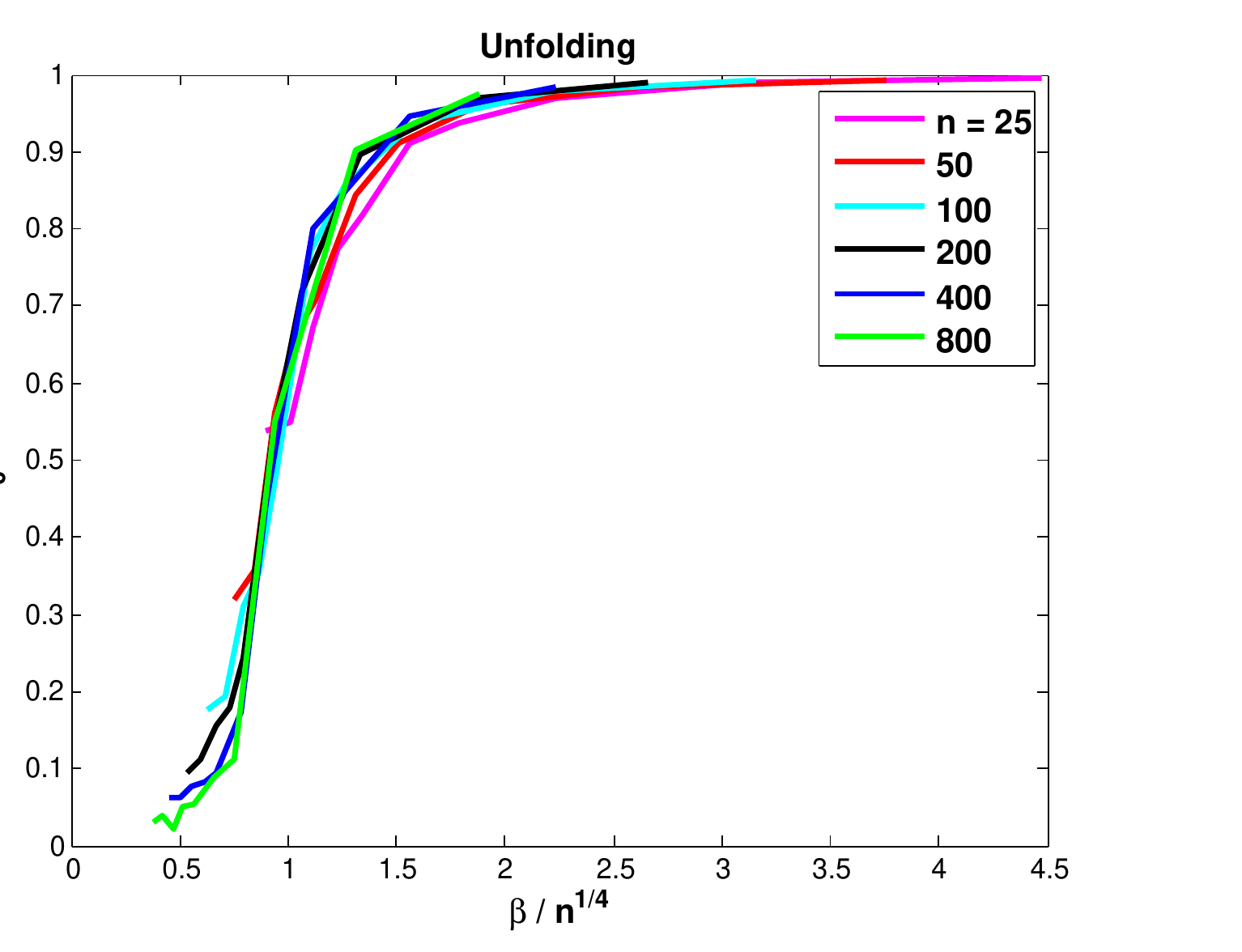}
\end{center}
\caption{Scaling with $n$ of the threshold signal-to-noise ratio
for different classes of algorithms. Left: tensor power iteration with
random initialization. Right: tensor unfolding.}\label{fig:ScalingComparison}
\end{figure}

In Figure \ref{fig:ScalingComparison} we compare the scaling with $n$
of the threshold signal-to-noise ratio for different type of
algorithms.
Our heuristic arguments suggest that tensor power iteration with
random initialization will work for $\beta\gtrsim n^{1/2}$, while
unfolding only requires $\beta\gtrsim n^{1/4}$ (our theorems  guarantee
this for, respectively, 
$\beta\gtrsim n$ and $\beta\gtrsim n^{1/2}$). We plot the average
correlation $|\<\hbv,\bvz\>|$ versus (respectively) $\beta/n^{1/2}$
and $\beta/n^{1/4}$. The curve superposition confirms that our
prediction captures the correct behavior already for $n$ of the order
of $50$.

\subsection{The value of side information}
\label{sec:SideInfo}
Our next experiment concerns a simultaneous matrix and tensor
PCA task: we are given a tensor  $\bX \in \otimes^3 \reals^{n}$ of
\ref{eq:SymmetricModel} with $k = 3$ and the signal to noise ratio
$\beta = 3$ is fixed.
 In addition, we observe $\bM = \lambda \bvz \bvz^\sT + \bN$ where
 $\bN \in \reals^{n \times n}$ is a symmetric noise matrix with upper
 diagonal  elements $i<j$ iid $\bN_{i,j} \sim \normal(0,1/n)$ and the
 value of $\lambda \in [0, 2]$ varies.
This experiment mimics a rank-1 version of topic modeling method
presented in \cite{anandkumar12} where $\bM$ is a matrix 
representing pairwise co-occurences and $\bX$ triples.

The analysis in previous sections suggest to use the leading
eigenvector of $\bM$ as the initial point of AMP algorithm for tensor
PCA on $\bX$. 
We performed the experiments on $100$ randomly generated  instances 
with $n = 50,200,500$ and report
in Figure \ref{fig:SimPCA} the mean values of $|\<\bvz,\hbv(\bX)\>|$
with confidence intervals.

Random matrix theory predicts $\lim_{n\to\infty}\<\hbv_1(M),\bvz\>
=\sqrt{1-\lambda^{-2}}$ \cite{feral2009largest}. Thus we can set
$\gamma = \sqrt{1-\lambda^{-2}}$ and apply the theory of the previous
section. In particular, Proposition \ref{propo:AMP_STATE} implies
\begin{equation*}
\lim_{n\to\infty}\<\hbv(\bX),\bvz\>= \beta \left ( 1/2  + \sqrt{ 1/4
    -1/\beta^{2}}\right )~~~\text{if}~~\gamma> \beta \left ( 1/2  -
  \sqrt{ 1/4 -1/\beta^{2}}\right )~~
\end{equation*}
and $\lim_{n\to\infty}\<\hbv(\bX),\bvz\> =0$ otherwise
Simultaneous PCA appears vastly superior to simple PCA. Our theory
captures this difference quantitatively already for $n=500$.
\begin{figure}
\begin{center}
\includegraphics[trim=4cm 0cm 0cm 0cm, clip=true,width = 16cm]{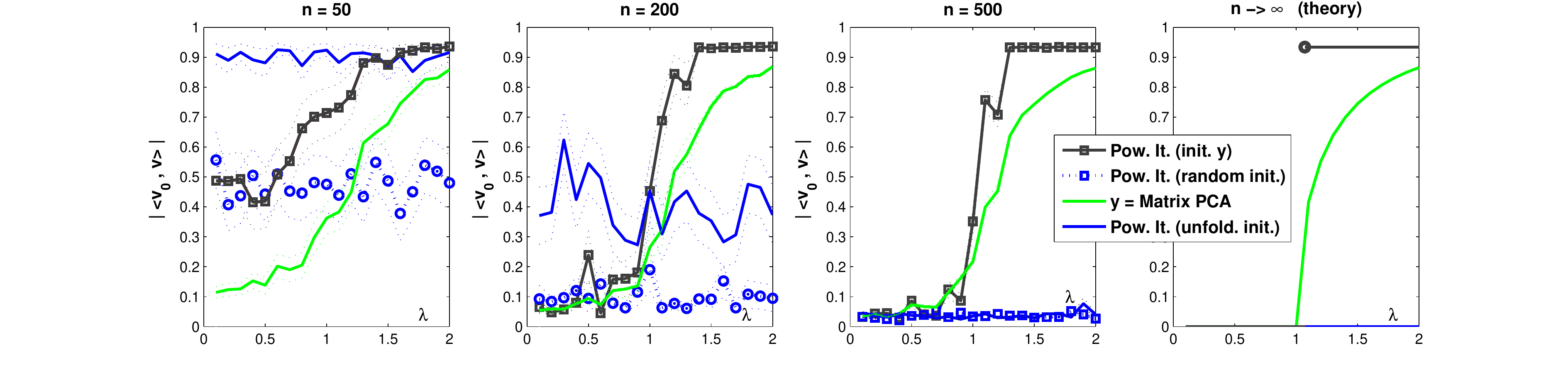}
\end{center}
\caption{Simultaneous PCA at $\beta = 3$. Absolute correlation of the estimated principal component  with the truth $|\<\hbv,\bvz\>|$, simultaneous PCA (black) compared with matrix (green) and tensor PCA (blue). }\label{fig:SimPCA}
\end{figure}

\section*{Acknowledgements}

This work was partially supported by the NSF grant CCF-1319979 and the grants AFOSR/DARPA
FA9550-12-1-0411 and FA9550-13-1-0036.

%

\appendix 

\section{Information theoretic bound: Proof of Theorem \ref{th:infoTheoretic}}
\label{app:ITBound}

Introduce the operator 
\[ \begin{aligned} \sU~~:& \otimes^k \reals^n \to \reals^{{n \choose k}}\\ & \bX \mapsto \sU(\bX)\end{aligned}~~,\]
where for indices $i_1<i_2<\cdots < i_k$, we have $\sU(\bX)_{a(i_1, \cdots, i_k)} = \bX_{i_1,\cdots, i_k}$ with $a(i_1, \cdots, i_k) = 1+\sum_{j=1}^k n^{j-1}(i_j-1)$. Let  $D(\cdot \Vert\cdot)$ denote
  the Kullback-Leiber divergence where $P_\bw$ is the law of $ \sU(\bX)$ conditional
  on $\bv_0=\bw$.  
\begin{lemma}\label{lem:Kullback} For any pairs of vectors $\bw,\bw' \in \bbS^{n-1}$ we have
  \begin{align*}
    D(P_\bw \Vert P_{\bw'}) \le 2 \frac{n}{k}\beta^2.
  \end{align*}
\end{lemma}
\begin{proof}
First note that for any $\bw \in \bbS^{n-1}$, $P_\bw$ is a Gaussian
probability distribution $$P_\bw = \normal \left ( \beta  \sU( \bw\ok) , \frac{1}{(k-1)! n}\id_{{n \choose k}} \right ) ~~.$$
On the other hand for any symmetric tensor $\bW \in \otimes^k \reals^n$ we have $k! \|\sU(\bW)\|_2^2\leq \|\bW\|_F^2$. Therefore we have
  \begin{align*}
    D(P_\bw \Vert P_{\bw'}) &= n(k-1)!  \beta^2  \|\sU(\bw\ok) - \sU(\bw'\ok)\|_2^2 \\
    &\le \frac{n}{k}\beta^2\|\bw\ok - \bw'\ok\|_F^2 \\
    &=  \frac{2 n}{k}\beta^2(1-\<\bw, \bw'\>^k) \\
    &\le  \frac{2 n}{k}\beta^2.
  \end{align*}

\end{proof}

We are now in position to prove Theorem \ref{th:infoTheoretic}.
  Let $\cV$ denote the class of estimators $\hbv$ with unit norm:
  \begin{equation} 
\cV = \left \{\begin{aligned} \hbv :& \otimes^k \reals^n \to
    \bbS^{n-1}\\ & \bX \mapsto \hbv(\bX) \end{aligned}\right \}~~.\
\end{equation}
\begin{proof}[Proof of Theorem \ref{th:infoTheoretic}]
Recall that the packing number
$N_n(\eps)$ of $\bbS^{n-1}$ is the maximum cardinality of any set
$\cN\subseteq \bbS^{n-1}$ such that $(\|\bx-\bx'\|_2\wedge\|\bx+\bx'\|_2)\ge \eps$ for any
$\bx,\bx'\in \cN$. By a standard argument, letting $M_n(\eps)$
the corresponding covering number\footnote{That is, the minimum
  cardinality of any set $\cN^*$ such that
  $\min_{\bx\in\cN^*}(\|\bu-\bx\|_2\wedge \|\bu+\bx\|_2 )\le\eps$ for any
  $\bu\in\bbS^{n-1}$.},  we have, for $\bx\in \bbS^{n-1}$ a point on the
unit sphere,
\begin{align}
N_n(\eps)\ge M_n(\eps)\ge \frac{{\rm Vol}_{n-1}(\bbS^{n-1})}{2{\rm
    Vol}_{n-1}(\bbS^{n-1}\cap B(\bx,\eps))} \ge
\left(\frac{1}{\eps}\right)^{n-1}\, .\label{eq:BoundPacking}
\end{align}
(Here ${\rm Vol}_{n-1}(\,\cdot\,)$ denotes the $(n-1)$-dimensional volume,
and $B(\bx,\eps)$ the ball of radius $\eps$ centered at $\bx$.)

  Let $\cN$ denote an $\eps$-packing with cardinality $|\cN|\ge
  N_n(\eps)$. 
Let $\bvz$ be uniformly distributed  in the set $\cN$. For an estimator $\hbv \in\cV$, we define
  $G(\hbv(\bX)) = \arg\min_{\bw\in\cN} \| \hbv(\bX) - \bw\|_2$. Consider the {\em error} event $\{ G(\hbv(\bX)) \ne \bvz\}$.
  By definition of $G(\hbv(\bX))$, the event $G(\hbv(\bX))\ne \bvz$
  implies $(\| \hbv(\bX) - \bv_0\|_2\wedge \|\hbv(\bX) + \bv_0\|_2) \ge \eps/2$.
 By Markov inequality we have:
\begin{align}
    \prob \{ G(\hbv(\bX)) \ne \bvz\} &\le \prob\big\{(\|\hbv(\bX) -
    \bv_0\|_2\wedge \|\hbv(\bX) -
    \bv_0\|_2) \ge \eps/2\} \nonumber \\
    &\le 4\frac{\E\{\Loss(\hbv,\bvz)\}}{\eps^2}\le \frac{4
      }{\eps^2} \inf_{\hbv\in\cV}\sup_{\bvz \in
        \bbS^{n-1}}\E\{\Loss(\hbv, \bvz)\}\, .\label{eq:Perrbound}
  \end{align}
  By Fano's inequality \cite{cover2012elements} we have that:
  \begin{align}\label{eq:FanoIneq}
    \prob\{G(\hbv(\bX)) \ne \bvz\} &\ge 1-\frac{{\rm I}(\bvz;\bX)+\log 2}{\log|\cN|}\\
&\ge 1 - \frac{\Delta + \log 2}{\log{|\cN|}}, 
  \end{align}
  where $\Delta = \max_{\bw\ne \bw'\in\cN}D(P_\bw \Vert P_{\bw'})$, and in
  the second inequality we used \cite{han1994generalizing}
\begin{align}
{\rm I}(\bvz;\bX) \le \frac{1}{|\cN|^2}\sum_{\bw\ne
    \bw'\in\cN}D(P_\bw \Vert P_{\bw'})\,.
\end{align}

Using Eq. ~(\ref{eq:BoundPacking})  and Lemma \ref{lem:Kullback}, 
in Eq.~(\ref{eq:FanoIneq}), we get
\begin{align}
   \prob\{G(\hbv(\bX)) \ne \bvz\} &\ge 1- \frac{2n\beta^2/k+\log
     2}{(n-1)\log(1/\eps)}\, .
\end{align}
Choosing $\eps=1/2$, we get, $\prob\{G(\hbv(\bX)) \ne \bvz\} \ge
1-(5\beta^2/k)$ for $n\ge 4$ and $\beta\le \sqrt{k/3}$.
In particular $\prob\{G(\hbv(\bX)) \ne \bvz\} \ge 1/2$ provided
$\beta<\sqrt{k/10}$. 
By Eq.~(\ref{eq:Perrbound}) this implies
\begin{align} 
\inf_{\hbv\in\cV}\sup_{\bvz \in
        \bbS^{n-1}}\E\{\Loss(\hbv, \bvz)\}\ge \frac{1}{32}\, .
\end{align}
\end{proof}

\section{Maximum likelihood: Proof
 Theorem \ref{th:nonconvexOpt}}
\label{sec:operatorNorm}
 
\subsection{Operator norm of the noise tensor: Proof
  of Lemma \ref{lem:BenArous}}
\label{sec:BenArous}

Let $\bZ_n \in\otimes^k\reals^n$ be a symmetric standard normal
tensor, and consider the associated objective function
\begin{align}
H_{\bZ}\; :& \;\;\bbS^{n-1}\to \reals\, ,\\
&\;\;\;\bv \mapsto H_{\bZ}(\bv) \equiv \<\bZ, \bv\ok\>\, .
\end{align}
While the function $H_{\bZ}(\,\cdot\,)$ is obviously non-convex, it
turns out that --for random data $\bZ$-- it is dramatically
so. Namely, it has an exponential number of local maximum, whose value
is --typically-- only a fraction of the value of the global maximum.

\begin{figure}
\begin{center}
\includegraphics[width=8cm]{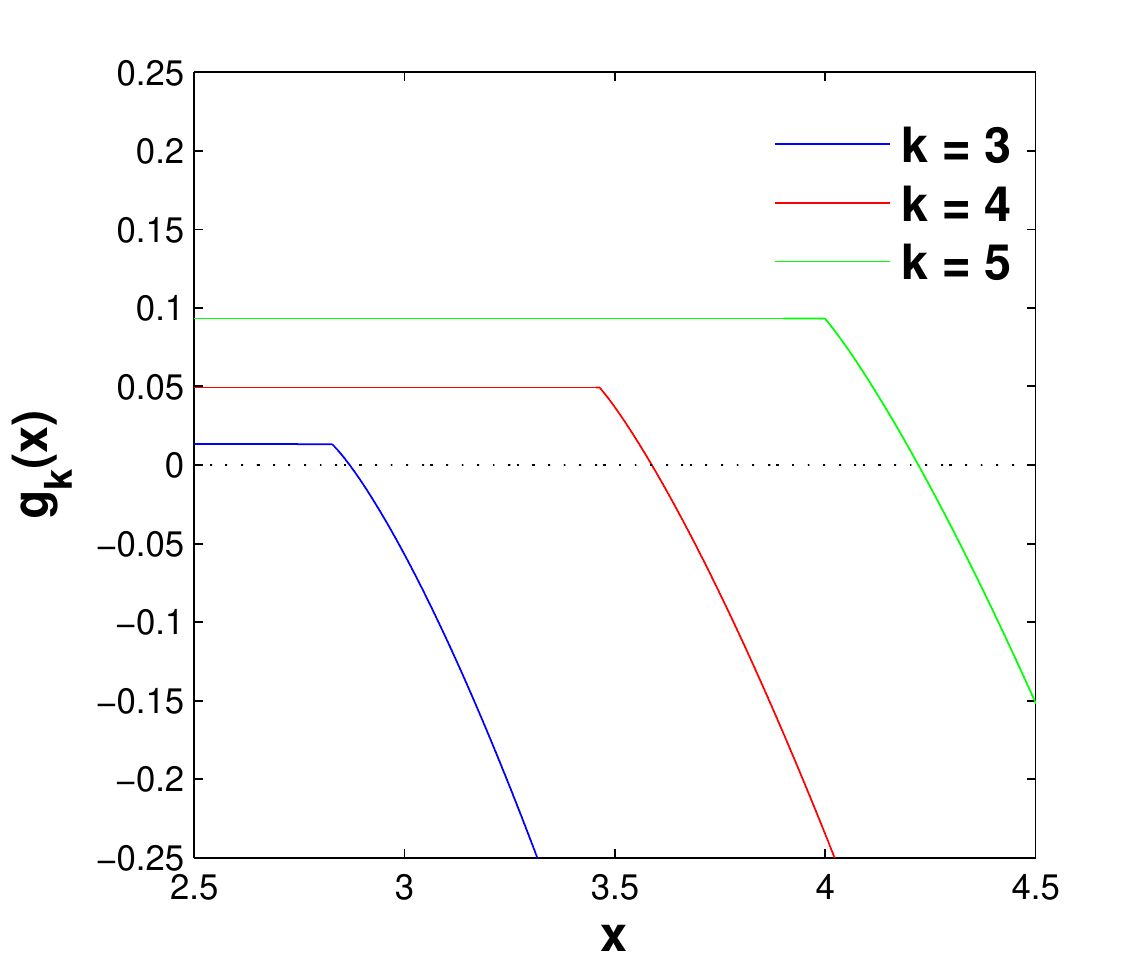}
\end{center}
\caption{The function $g_k(x)$ defined in
  Eq.~(\ref{eq:ComplexityFormula}). As proved in  \cite{Auffinger13},
  the expected number of local maxima of
the objective function $H_{\bZ}(\bv)$ is --to leading exponential
order--  $\exp\{n g_k(x)\}$.}\label{fig:ComplexityFunction}
\end{figure}
In order to quantify this phenomenon, for  $x\in \reals$, let  ${\sf
  C}_k(\bZ_n,x)$ denote the number of 
local maxima of $H_{\bZ}(\,\cdot\,)$
over $\bbS^{n-1}$, that have value larger or equal than  $x$. 
The next Lemma from \cite{Auffinger13} characterizes the growth rate
of the number of local minima.
 \begin{theorem}[Theorem 2.5 and Lemma 6.3 in \cite{Auffinger13}]
For any $k\ge 3$, we have
 \begin{align}\label{eq:AuffingerLimit} 
\lim_{n \to \infty} \frac 1n \log\, \E\, {\sf C}_k(\bZ_n, x) =
g_k(x)~~,
\end{align}
where, for $x\ge \eta_k\equiv 2\sqrt{k-1}$ 
\begin{align}
g_k(x) = \frac{1}{2} \left \{ \frac{2-k}{k} - \log \left
    (\frac{kz^2}{2} \right ) + \frac{k-1}{2}z^2 - \frac{2}{k^2z^2}
\right \}~~,~~
z = \frac{1}{(k-1)\sqrt{2k}}\left (   x  -
  \sqrt{x^2  - 4(k-1) }\right )~~. \label{eq:ComplexityFormula}
\end{align}
Further, for $x<\eta_k$, $g_k(x)= g_k(\eta_k)$.
 \end{theorem}
The function $g_k(x)$ is monotone decreasing  for $x\ge \eta_k$, and
non-negative if and only if $x\in [\eta_k,\mu_k]$ for some $\mu_k>0$
(strictly positive for $x\in [\eta_k,\mu_k)$).
In Figure \ref{fig:ComplexityFunction}, we plot $g_k(x)$ for $k\in\{3,4,5\}$. 
Informally, this means that the function $H_{\bZ}(\bv)$ has
exponentially many local maxima with value $H_{\bZ}(\bv)\approx x$ for
any $x\in  [\eta_k,\mu_k)$. To leading exponential order, the number
of such maxima is given by $\exp\{n\, g_k(x)\}$.

The value  $\mu_k$ can be determined as the unique solution to the equation $g(x)
= 0$. It is immediately to do this numerically, obtaining the values in
Section \ref{sec:infoTheoryandUpperBounds}.

The last result implies that the global maximum of $H_{\bZ}(\bv)$ is
(asymptotically) at least $\mu_k$. Indeed the global maximum is
necessarily a local maximum as well. The next result implies that
indeed the global maximum converges to $\mu_k$.
 \begin{theorem}[Theorem 2.12 in \cite{Auffinger13}]\label{cor:mukProbabilist} 
Let $\mu_k$ denote the   unique non-negative root of the equation
$g_k(x) = 0$, for $x\ge \eta_k\equiv 2\sqrt{k-1}$.
Then
\begin{align}
\lim_{n\to\infty}\E\| \bZ \|_{op} =\mu_k  \, .
\end{align}
\end{theorem}

In order to derive the large-$k$ asymptotics of $\mu_k$, we
rewrite Eq.~(\ref{eq:ComplexityFormula}) in terms of the variable
$y\equiv k^2z^2/2$. We get $g_k(x) = f_k(y(x))/2$, where
\begin{align}
f_k(y) =\frac{2-k}{k} +\log(k)- \log (y)+ \frac{k-1}{k}\, y -
\frac{1}{y}\, ,\;\;\;\;
x =  (k-1) \, \sqrt{\frac{y}{k}}+\sqrt{\frac{k}{y}}\, .
\end{align}
Further $y \in (0, k/(k-1)]$.  The claimed asymptotics follows by
showing that the only solution of $f_k(y) = 0$ in this interval obeys
$y_k = (\log k)^{-1} (1+o_k(1))$. This in turns can be showed by using the
bounds
\begin{align}
\log(ke^{-1+(2/k)}) -\log y- \frac{1}{y} \le f_k(y) \le \log(ke^{2/k})
-\log(y) -\frac{1}{y}\, ,
\end{align}
and showing that the solution of $y^{-1}+\log(y) = \log(a)$ for large
$a$ is $y^{-1} = a +\Theta(\log(a))$.

Finally, the norm $\|\bZ\|_{op}$ concentrates tightly around its expectation.
\begin{lemma}
For any $s\ge 0$, we have
\begin{align}
\prob\big(\big|\|\bZ\|_{op}-\E\|\bZ\|_{op}\big|\ge s\big)\le 2\,
e^{-ns^2/(2k)}\, .
\end{align}
\end{lemma}
 \begin{proof}
Note that 
\begin{align}
\<\bZ,\bv^{\otimes k}\> = \sqrt{\frac{k}{n}}\, \<\bG,\bv^{\otimes k}\>
\end{align}
is a Lipschitz function with Lipschitz modulus $\sqrt{k/n}$ (with
respect to Euclidean norm) of the Gaussian  vector (tensor)
$\bG$. Hence $\|\bZ\|_{op}$ is Lipchitz continuous with the same
modulus. The claim follows from Gaussian isoperimetry \cite{Ledoux}.
 \end{proof}
 \begin{remark}
Note  that to make the connection with the notations used in
\cite{Auffinger13}, one has to use the proper scaling 
$H_{n,k}(\bv) = \frac n {\sqrt k} \, L_\bZ(\bv/\sqrt{n})$
($H_{n,k}(\bv)$is the objective function considered in \cite{Auffinger13}).
\end{remark}

\begin{remark}
The upper bound on the tensor operator norm obtained from
Sudakov-Fernique inequality is not tight. In fact taking $\beta = 0$
in Lemma \ref{lem:upperBoundX} gives the loose upper bound
$\|\bZ\|_{op}\leq k$. Except in the case of random matrices ($k=2$),
this is loose roughly by a factor $\sqrt{k}$.
\end{remark}

\subsection{Proof of
 Theorem \ref{th:nonconvexOpt}}
\label{sec:ProofML}

By optimality of $\hbv$, we have
\begin{align}
\beta \< \bvz, \hat \bv\>^k + \< \bZ, \hat \bv\ok\> \geq \beta + \<
\bZ , \bvz\ok\>~, 
\end{align}
whence
\begin{align}
 \< \bvz, \hat \bv\>^k &\geq 1-\frac 1
\beta \< \bZ ,\hat \bv\ok - \bvz\ok\>\\
&\ge  1-\frac 1
\beta \Big(\|\bZ\|_{op} - \<\bZ,\bvz\ok\>\Big)\, .
\end{align}
Note that $\|\bZ\|_{op} - \<\bz,\bvz\ok\>$ is Lipchitz continuous in
the Gaussian random variables $\bG$, with modulus bounder by
$2\sqrt{k/n}$.
Hence, by Gaussian isoperimetry, with probability at least
$1-2e^{-ns^2/(8k)}$, we have (since $\bZ$ and $\bvz$ are independent, $\E\< \bZ, \bvz\ok\> = 0$)
\begin{align}
 \< \bvz, \hat \bv\>^k \ge 1-\frac{1}{\beta}\big(\E\|\bZ\|_{op}+s\big)
\end{align}
Using $(1-\alpha)^{1/k} \geq (1-\alpha)$ which holds for $\alpha \in
[0,1]$,  and rescaling $s$,
we get $|\< \bvz, \hat \bv\>|\geq 1-(   \mu_k+s)/ \beta $ with
probability at least $1-2e^{-ns^2/(16k)}$ for all $n$ large enough.

\subsection{Proof of Lemma \ref{lem:upperBoundX}}
\label{sec:ProofUpperBoundX}

\begin{lemma}\label{lem:FGinnerproducts} 
For each $n \in \naturals$, let $\bg \sim  \normal(0,\id_n /n)$ and $\bvz(n)\in
\reals^n$ be a vector with $\|\bvz(n)\|_2=1$.
Then there exists a sequence $\delta_n$ independent from $x$, such
that $\lim_{n \to \infty} \delta_n = 0$ and the following happens.
With probability one, there exists (a random) $n_0$ such that, for all
$n\ge n_0$, 
\begin{equation}
\sup_{\tau\in [0,\tau_{\rm max}] }\Big|~ \|\bg + \tau \bvz\|_2 -
\sqrt{1+\tau^2} \Big | \leq \delta_n\, .
\end{equation}
\end{lemma}
\begin{proof}
Since $x\mapsto\sqrt{x}$ is uniformly continuous on bounded intervals
$[0,M]$, it is sufficient to prove
\begin{equation}
\sup_{\tau\in [0,\tau_{\rm max}] }\Big|~ \|\bg + \tau \bvz\|_2^2-
(1+\tau^2) \Big | \leq \delta_n\, .
\end{equation}
for all $n\ge n_0$, and an eventually different sequence $\delta_n$. 
By triangular inequality and using $\|\bvz\|_2 = 1$, 
\begin{align}
\sup_{\tau\in [0,\tau_{\rm max}] }\Big|~ \|\bg + \tau \bvz\|_2^2-
(1+\tau^2) \Big | \le \big|\|\bg\|^2-1\big|+2\tau_{\rm
  max}\big|\<\bvz,\bg\>\big|
\end{align}
Next we have $\lim_{n\to\infty}\|\bg\|^2=1$ almost surely by the
strong law of large numbers, and $\<\bvz,\bg\>\sim \normal(0,1/n)$
whence $\lim_{n\to\infty}\<\bvz,\bg\> =0$ by Borel-Cantelli. 
\end{proof}

\begin{proof}[Proof of Lemma \ref{lem:upperBoundX}]
For $\kappa\in [0,1]$, we define
\begin{align}
\cW_{\kappa} &\equiv \left \{ \bv \in \bbS^{n-1}~:~ 
  \<\bv,\bvz\> =\kappa\right \}\,, \\
M_{\bX}(\kappa) & \equiv\max\big\{\<\bX, \bv^{\otimes k}\>:\;\, 
\bv\in\cW_{\kappa}\big\}\, ,\\
\oM(\kappa) & \equiv\E M_{\bX}(\kappa) = \E \max\big\{\<\bX,\bv^{\otimes k}\>:\;\, 
\bv\in\cW_{\kappa}\big\}\, .
\end{align}
Note that 
\begin{align}
\lambda_1(\bX) & = \max_{\kappa\in [0,1]}M_{\bX}(\kappa) . \label{eq:SigmaPlus}
\end{align}
The function  $\bX\mapsto M_{\bX}(\kappa)$ is a Lipschitz continuous
function with
Lipschitz constant  $\sqrt{k/n}$ of the standard 
Gaussian tensor  $\bG$ (namely $|M_{\bX}(\kappa)-M_{\bX'}(\kappa)|\le
(k/n)^{1/2}\|\bG-\bG'\|_F$). Hence, by Gaussian isoperimetry, we have 
\begin{align}
\prob \Big\{ \big|M_{\bX}(\kappa) - \oM(\kappa)\big| \ge t \Big\}
  \le 2\, 
  e^{-n t^2/(2k)}\, . \label{eq:Isoperimetry}
\end{align}
Further we claim that $\kappa\mapsto M_{\bX}(\kappa)$ is uniformly
continuous for $\kappa\in [0,1]$. In order to prove this, let 
\begin{align}
\bv(\kappa) = \kappa \bvz + \sqrt{1-\kappa^2} \bv^\perp = \text{argmax} \big\{\<\bX, \bv^{\otimes k}\>:\;\, 
\bv\in\cW_{\kappa}\big\}~~,
\end{align}
where $\<\bv^\perp,\bvz\> = 0 $. We have, for $\kappa_1,\kappa_2 \in
[0,1]$, and by letting 
$\bv^\perp$ and $\bw^\perp$ denote the  perpendicular components
of $\bv(\kappa_1)$ and $\bv(\kappa_2)$, we have for some constant $c>0$
\begin{align}
M_\bX(\kappa_1) = & \< \bX , \bv(\kappa_1)\ok \> \nonumber \\
& = \< \bX ,\left (\kappa_1 \bvz + \sqrt{1-\kappa_1^2} \bv^\perp\right )\ok \> \nonumber\\
& \geq   \< \bX ,\left (\kappa_1 \bvz + \sqrt{1-\kappa_1^2} \bw^\perp\right )\ok \>  \quad\text{by optimality} \nonumber\\
& = \< \bX , \left \{ \kappa_2 \bvz + \sqrt{1-\kappa_2^2} \bw^\perp + (\kappa_1 - \kappa_2) \bvz + \left ( \sqrt{1-\kappa_1^2} - \sqrt{1-\kappa_2^2} \right )\bw^\perp \right \}\ok \> \nonumber\\
& = M_\bX(\kappa_2) + \sum_{q = 1}^k { k \choose q}\<\bX,
\bv(\kappa_2)^{\otimes q} \otimes \left \{(\kappa_1 - \kappa_2)\bvz + (\sqrt{1-\kappa_1^2} - \sqrt{1-\kappa_2^2})\bw^\perp \right \}^{\otimes (k-q)} \>\label{ineq:kchooseqFormula} \\
&\geq  M_\bX(\kappa_2)  - c \|\bX\|_{op} \left \{ (\kappa_1 - \kappa_2)^2+ \left (\sqrt{1-\kappa_1^2} - \sqrt{1- \kappa_2^2}\right )^2 \right \}^{1/2}~~. \label{ineq:sqrtKappa1Kappa2Term}
\end{align}
where Eq. (\ref{ineq:kchooseqFormula}) was obtained by exploiting the symmetry of the tensor $\bX$ and Eq. (\ref{ineq:sqrtKappa1Kappa2Term}) was derived using the norm of the vector $\left \{(\kappa_1 - \kappa_2)\bvz + (\sqrt{1-\kappa_1^2} - \sqrt{1-\kappa_2^2})\bw^\perp \right \}$.
Using Eq.~(\ref{eq:Isoperimetry}) over a grid $\kappa\in
\{0,1/n,2/n,\dots, 1-1/n, 1\}$, and the fact\footnote{This follows
  from Lemma \ref{lem:BenArous} and triangular inequality, or from a
  standard $\eps$-net argument.} that $\|\bX\|_{op}\le C$
for some constant $C>0$ with probability $1-e^{-\Theta(n)}$, we have for all $t > 0$
and some constant $c>0$
\begin{align}
\prob \Big\{\max_{\kappa\in [0,1]} \big|M_{\bX}(\kappa) - \oM(\kappa)\big| \ge t \Big\}
  \le 2\, n
  e^{-n c t^2/2} + 2\, e^{-cn}\, . 
\end{align}
In particular, by Borel-Cantelli we have, almost surely,
\begin{align}
\lim_{n\to\infty}\max_{\kappa\in [0,1]} \big|M_{\bX}(\kappa) -
\oM(\kappa)\big| =0\, . \label{eq:MoM}
\end{align}

In order to upper bound $\oM(\kappa)$, we apply Sudakov-Fernique inequality for non-centered Gaussian processes \cite[Theorem 1]{vitale2000some} to the two processes $\{\cX_\bv\}$, $\{\cY_\bv\}$ indexed by
$\bv \in \cW_{\kappa}$ defined as follows:
\begin{align}
\cX_\bv & \equiv \<\bX, \bv^{\otimes \bv} \>=  \beta \langle \bvz, \bv\rangle^k+\< \bZ, \bv^{\otimes k} \>\, ,\\
\cY_\bv & \equiv \beta\langle \bvz, \bv\rangle^k+  k \langle \bg , \bv\rangle \, ,
\end{align}
for random a vector $\bg \sim \normal(0, \id_n / n)$. It is easy to see that $\E\cX_\bv = \E\cY_\bv$ and 
\begin{align}
\E\big\{\big[\cX_\bv-\cX_\bw \big]^2\big\}&  = \left \{ \E\cX_\bv - \E \cX_\bw\right \}^2 + 
\frac{2 ~k }{n}
\big(1-\<\bv,\bw\>^k\big)\, ,\\
\E\big\{\big[\cY_\bv-\cY_\bw\big]^2\big\}&  =   \left \{ \E\cY_\bv - \E\cY_\bw \right \} ^2 + \frac{2~k^2 }{n}
\big(1-\<\bv,\bw\>\big)\, .
\end{align}

Hence $\E\big\{\big[\cX_\bv-\cX_\bw\big]^2\big\}\le
\E\big\{\big[\cY_\bv-\cY_\bw\big]^2\big\}$
(this follows from $1-a^k\le k(1-a)$ for $a\in [-1,1]$). We conclude
that
\begin{align}
\oM(\kappa)
&\le\E \max \left \{  \beta \kappa^k +\frac k {\sqrt n}    \langle \bg , \bv\rangle ~:~\bv \in \cW_{\kappa}
\right \}\\
&\le\max_{\tau\geq 0} \Big\{\left (\frac {\tau}{\sqrt {1+\tau^2}} \right )^k + \frac{k}{\sqrt{1+\tau^2}} \Big\}+ \delta_n
\end{align}
where $\tau= \kappa / \sqrt{1-\kappa^2}$ and $\delta_n$ satisfies
$\lim_{n\to \infty} \delta_n = 0$ uniformly over $\kappa \in [0,1]$,
by Lemma \ref{lem:FGinnerproducts}.
We finally conclude that 
\begin{align}
 \lim \sup_{n \to \infty} \E \|\bX\|_{op} \leq\max_{\tau\geq 0} \left
   (\frac{\tau}{\sqrt {1+\tau^2}} \right )^k +
 \frac{k}{\sqrt{1+\tau^2}}~~.
\end{align}
Concentration around the expectation follows by Gaussian isoperimetry as in the proof of Lemma \ref{lem:BenArous}.
\end{proof}

\section{Power Iteration: Proof of Theorem \ref{lem:powerIterationPositiveResult} }

Let $\tau_t \equiv \<\bvz,\bv^t\>$ and $\xi\equiv
\|\bZ\|_{op}/\beta$. 
Let $\tau_{\rm min}$, $\tau_{*}\in [0,1]$ be the two solutions
of 
\begin{align}
\tau^{k-1}(1-\tau) = \xi\, .
\end{align}
We will show below that  our assumptions imply $\tau_0>\tau_{\rm min}$. Further
$\tau\ge \tau_{\rm min}$ implies 
$\tau^{k-1}-\xi\ge 0$.

By 
definition of $\bX$, we have
\begin{align}\label{eq:mothereq}
\bX\{ \bv^t\} = \beta \tau_t^{k-1} \bvz + \bZ\{\bv^t\} \,,
\end{align}
which implies, by triangular inequality,
\begin{align}
\<\bvz,\bX\{ \bv^t\}\> \ge  \beta \tau_t^{k-1} -\beta\xi \,, \label{eq:mothereq2}\\
\|\bX\{ \bv^t\}\|_2 \le   \beta \tau_t^{k-1}  +\beta\xi  \, . \label{eq:mothereq3}
\end{align}
We will prove the first inequality  $\tau_t \geq \tau_{\rm min}$ by induction. It is true at $t = 0$ by assumption.
Assume it is true at $t$. Then $\tau_{t+1}\geq 0$ using Eq.~(\ref{eq:mothereq2}).

Hence we can divide the two inequalities  above 
obtaining $\tau_{t+1} \ge (\tau_t^{k-1}-\xi)/(\tau_t^{k-1}+\xi)$ which implies
\begin{align}
\tau_{t+1}
&\ge 1 -\frac{\xi}{\tau_t^{k-1}}\, .
\end{align}
In particular $\tau_t\ge \ttau_t^0$ where the latter sequence
is defined by $\ttau_{t+1} = f(\ttau_t)$, $\ttau_0 = \tau_0$, for  
$f(x) = 1-\xi\, x^{-k+1}$. The function $f(\, \cdot\, )$ is concave
and monotone increasing, and   maps $[\tau_{\rm min},\tau_*]$
into itself,  with $f'(\tau_{\rm min}) >1$, $f'(\tau_*)<1$. By
standard calculus argument  $\ttau_t\to \tau_*$ exponentially fast,
which implies
\begin{align}
\<\bv^t,\bvz\> \ge \tau_* - c_0 \, e^{-t/c_0}\, .
\end{align}

To conclude the proof of Eq.~(\ref{eq:BoundPI}), we notice that, for $\xi\le 1/(2e(k-1))$
\begin{align}
\tau_* & > 1- e\, \xi\, ,\\
\tau_{\rm min}& < [(k-1)\xi]^{1/(k-1)}\, ,
\end{align}
where we recall that $\tau_{\rm min}$, $\tau_*$ are the two solutions
of $g_k(x) \equiv x^{k-1}(1-x) = \xi$ in the interval $[0,1]$. 
For the first inequality, note that, in the interval
$[e^{-1/(k-1)},1]$, $g_k(x)$ is decreasing with $g_k(x)\ge
e^{-1}(1-x)$.
This implies 
\begin{align}
e^{-1}(1-\tau_*) \ge \xi
\end{align}
i.e. $\tau_*\ge 1-e\,\xi$ as long as $1-e~ \xi\ge  e^{-1/(k-1)}$,
which is implied by $\xi\le 1/(2e(k-1))$.

For the second inequality, note that, in the interval
$[0,1-(k-1)^{-1}]$,
we have $g_k(x)$ increasing with $g_k(x) \ge x^{k-1}/(k-1)$. This implies
\begin{align}
\frac{\tau_{\rm min}^{k-1}}{k-1}\ge \xi\, ,
\end{align}
as long as $ [(k-1)\xi]^{1/k-1}\le 1-(k-1)^{-1}$,  which follows,
again, by our assumptions.

Finally, conditions (\ref{eq:SNR_PowerIteration_BIS}),
(\ref{eq:InitPowerIteration_BIS})
follow directly by applying Lemma \ref{lem:BenArous}.
%
%
%
\section{Approximate Message Passing: Proof of Theorem \ref{thm:AMPisGood}}

\begin{figure}
\begin{center}
\includegraphics[trim=1cm 0cm 0cm 0cm, clip=true,width =
7cm]{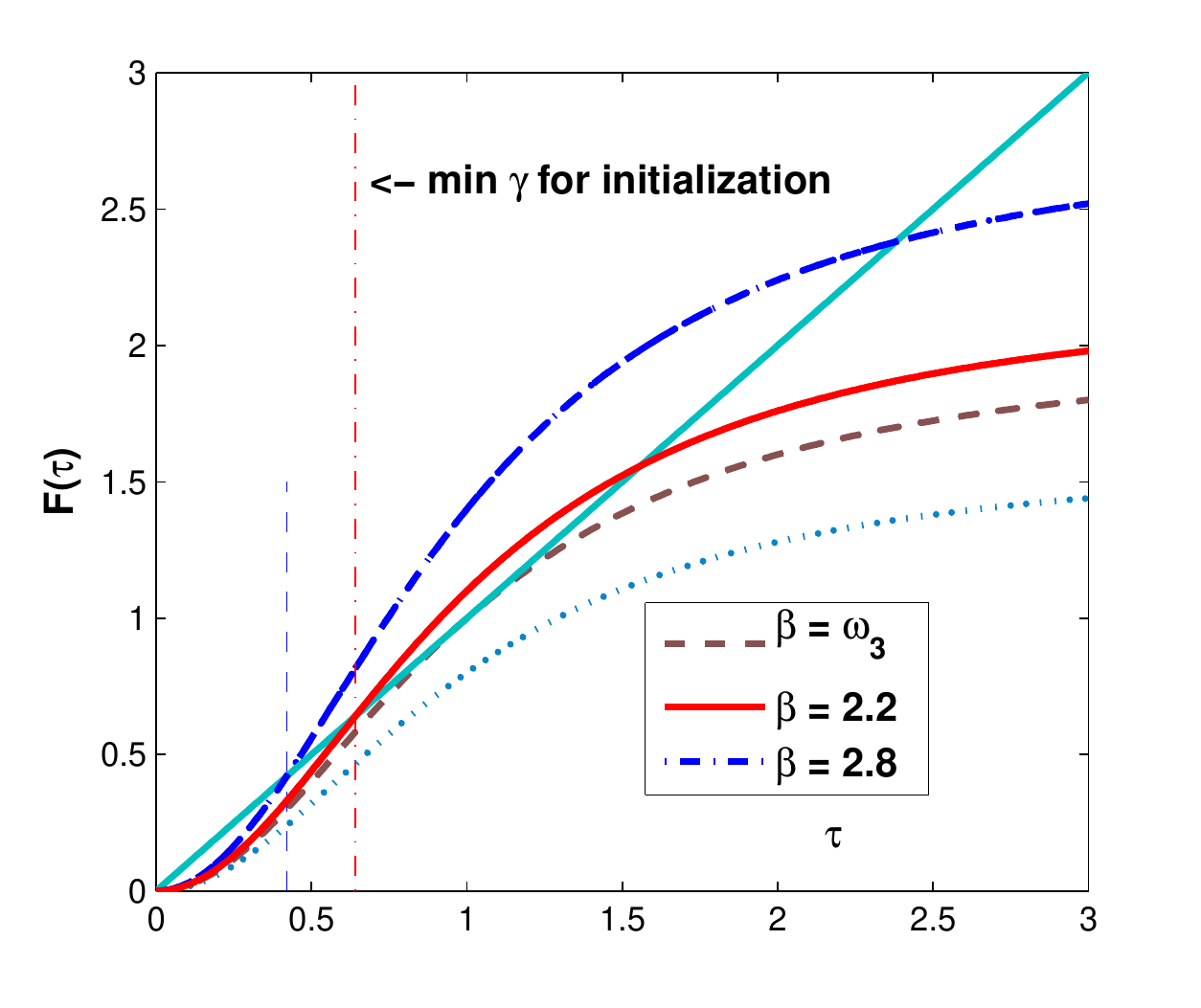}
\end{center}
\caption{Right: iteration functions versus $\tau$. The limit curve for
  $\beta = \omega_3 = 2$ separates the two types of behaviors: a
  non-zero fixed point exists for $\beta\ge \omega_k,$ and does not
  exists for $\beta<\omega_k$.}\label{fig:illustrationTheory}
\end{figure}

Let us recall the state evolution recursion (\ref{eq:SE}) 
\begin{align}
\tau_{t+1}^2 &= f(\tau_t^2;\beta)\, ,\\
f(z;\beta) &\equiv \beta^2\Big(\frac{z}{1+z}\Big)^{k-1}\, .
\end{align}
Notice that $f(\,\cdot\,;\beta)$ is strictly positive and monotone
increasing on $\reals_{>0}$. The theorem follows by proving that the
following claims hold for $\beta>\omega_k$
\begin{enumerate}
\item The fixed point equation $\tau^2 = f(\tau^2;\beta)$ has two
 strictly positive solutions
$\tau_1^2(\beta)<\tau_2^2(\beta)$.
\item The smallest fixed point is given by $\tau_1(\beta) =
  \sqrt{1/\epsilon_k(\beta)-1}$
as in the statement.
\item The largest fixed point satisfies $\tau_2(\beta)
  >1-(2/\beta^2)$.
\end{enumerate}
The behavior of the function $f(\tau^2;\beta)$ is illustrated in
Fig.~\ref{fig:illustrationTheory}.

In order to prove the above statements, it is convenient to use the
monotone parametrization
$x \equiv \tau^2/(1+\tau^2)$ that maps $\reals_{\ge 0}$ onto the
interval $[0,1)$. After some algebra (and discarding the solution at $x=0$),  fixed point equation then
reads 
\begin{align}
\frac{1}{\beta^2} = x^{k-2}(1-x) \equiv h_k(x)\, .\label{eq:FixedPointSimpler}
\end{align}
Now, the function $x\mapsto h_k(x)$ is continuously differentiable and
strictly positive in the interval $(0,1)$, with $h_k(0) =
h_k(1)=0$. Further, simple calculus shows it has a unique stationary
point (a maximum) at
$x_* = (k-2)/(k-1)$ with $h_k(x_*) = 1/\omega_k^2$. This implies that,
for $\beta>\omega_k$,
Eq.~(\ref{eq:FixedPointSimpler}) has two fixed points
$0<x_2(\beta)<x_*<x_1(\beta)<1$ thus implying points 1 and 2 above
(the latter immediately follows from inverting the
re-parametrization).  

To prove point 3, note that
\begin{align}
x_2 &= 1 -\frac{1}{x_2^{k-2}\beta^2} \\
&\ge 1- \frac{1}{x_*^{k-2}\beta^2} \\
& \ge 1-\frac{3}{\beta^2}
\end{align}
where the second inequality follows since $x_2>x_*$.
By state evolution (Proposition  \ref{propo:AMP_STATE}),
together with the fact that $\tau_t\to\tau_2$,
we have
\begin{align}
\lim_{t \to \infty} \lim_{n \to \infty}  \Loss(\bv^t /
\|\bv^t\|_2,\bvz) &= 2 \Big(1-\sqrt{\frac{\tau_2^2}{1+\tau_2^2}}\Big) \\
&= 2(1-\sqrt{x_2})\le 2(1-x_2) \le \frac{6}{\beta^2}\,.
\end{align}

\bibliographystyle{amsalpha}

\newcommand{\etalchar}[1]{$^{#1}$}
\providecommand{\bysame}{\leavevmode\hbox to3em{\hrulefill}\thinspace}
\providecommand{\MR}{\relax\ifhmode\unskip\space\fi MR }
\providecommand{\MRhref}[2]{%
  \href{http://www.ams.org/mathscinet-getitem?mr=#1}{#2}
}
\providecommand{\href}[2]{#2}

\end{document}